\documentclass[a4paper,11pt]{article}
\usepackage{fullpage} 

\usepackage{diagbox} 
\usepackage{slashbox} 
\usepackage{tabularx} 
\usepackage{bbm} 
\usepackage{pifont}  

\usepackage{amsmath}
\usepackage{amssymb}
\usepackage{amsthm}
\usepackage{mathtools}
\usepackage{mathdots}
\usepackage{appendix}
\usepackage{graphicx}
\usepackage{enumerate}
\usepackage{microtype}
\usepackage{mleftright}
\usepackage{mdframed}
\usepackage{authblk}
\usepackage{tcolorbox}
\usepackage[english]{babel}
\usepackage[utf8]{inputenc}
\usepackage{algorithm}
\usepackage[noend]{algpseudocode}
\usepackage{xspace}
\usepackage{tikz}
\usepackage{pgfplots}
\pgfplotsset{width=8cm,compat=1.9}
\usepackage{mathdots}
\usepackage{url}
\usepackage{forest}
\forestset{default preamble={for tree={s sep+=1cm}}}
\usepackage{caption,subcaption}
\usepackage[noadjust]{cite}
\usepackage{hyperref}

\tcbset{
    rounded corners,
    colback = white,
    before skip = 0.2cm,
    after skip = 0.5cm,
    boxrule = 1.5pt,
    arc = 5pt
}

\makeatletter
\renewcommand{\ALG@name}{Protocol}
\makeatother
\allowdisplaybreaks

\graphicspath{ {./figures/} }

\usepackage{algorithm}
\usepackage[noend]{algpseudocode}

\DeclarePairedDelimiter{\ceil}{\lceil}{\rceil}
\DeclarePairedDelimiter{\floor}{\lfloor}{\rfloor}

\newtheorem{theorem}{Theorem}[section]
\newtheorem*{theorem*}{Theorem}
\newtheorem{proposition}[theorem]{Proposition}
\newtheorem{lemma}[theorem]{Lemma}
\newtheorem{corollary}[theorem]{Corollary}
\newtheorem*{corollary*}{Corollary}

\newtheorem{observation}[theorem]{Observation}

\theoremstyle{definition}
\newtheorem{definition}[theorem]{Definition}

\newcommand{\cU}{\mathcal{U}}
\newcommand{\cH}{\mathcal{H}}
\newcommand{\cX}{\mathcal{X}}

\newcommand{\cS}{\mathcal{S}}

\newcommand{\cE}{\mathcal{E}}

\newcommand{\cB}{\mathcal{B}}

\newcommand{\cI}{\mathcal{I}}
\newcommand{\N}{\mathbb{N}}

\newcommand{\ind}{\mathbbm{1}}

\newcommand{\TWDG}{\mathtt{TWDG}}
\newcommand{\TWWDG}{\mathtt{TW^2DG}}

\newcommand{\rand}{\operatorname{rand}}

\newcommand{\cA}{\mathcal{A}}

\newcommand{\Lrn}{\mathsf{Lrn}}

\newcommand{\SOA}{\mathsf{SOA}}

\newcommand{\RealizableExpAT}{\mathsf{RealizableExpAT}}
\newcommand{\RealizableConceptAT}{\mathsf{RealizableConceptAT}}
\newcommand{\NarrowConceptAT}{\mathsf{NarrowConceptAT}}
\newcommand{\ExpAT}{\mathsf{ExpAT}}
\newcommand{\ConceptAT}{\mathsf{ConceptAT}}
\newcommand{\DTExpAT}{\mathsf{DTExpAT}}

\newcommand{\LD}{\mathtt{L}}

\newcommand{\w}{\mathtt{w}}
\newcommand{\dw}{\mathtt{dw}}

\newcommand{\M}{\mathtt{M}}
\newcommand{\tree}{\mathbf{T}}

\newcommand{\Mis}{\M^\star}

\DeclareMathOperator*{\argmax}{arg\,max}

\newif\ifanonymous
\anonymousfalse

\begin{document}
\title{Deterministic Apple Tasting}

\author[1]{Zachary Chase}
\author[2]{Idan Mehalel}
\affil[1]{Faculty of Mathematics, Technion, Israel}
\affil[2]{The Henry and Marilyn Taub Faculty of Computer Science, Technion, Israel}

\maketitle

\begin{abstract}
    In binary ($0/1$) online classification with apple tasting feedback, the learner receives feedback only when predicting $1$. Besides some degenerate learning tasks, all previously known learning algorithms for this model are randomized. Consequently, prior to this work it was unknown whether deterministic apple tasting is generally feasible. In this work, we provide the first widely-applicable deterministic apple tasting learner, and show that in the realizable case, a hypothesis class is learnable if and only if it is deterministically learnable, confirming a conjecture of \cite{raman2024apple}. Quantitatively, we show that every class $\cH$ is learnable with mistake bound $O \mleft(\sqrt{\LD(\cH) T \log T} \mright)$ (where $\LD(\cH)$ is the Littlestone dimension of $\cH$), and that this is tight for some classes. This demonstrates a separation between a deterministic and randomized learner, where the latter can learn every class with mistake bound $O \mleft(\sqrt{\LD(\cH) T} \mright)$, as shown in \cite{raman2024apple}.

    We further study the agnostic case, in which the best hypothesis makes at most $k$ many mistakes, and prove a trichotomy stating that every class $\cH$ must be either easy, hard, or unlearnable. Easy classes have (both randomized and deterministic) mistake bound $\Theta_{\cH}(k)$. Hard classes have randomized mistake bound $\tilde{\Theta}_{\cH} \mleft(k + \sqrt{T} \mright)$, and deterministic mistake bound $\tilde{\Theta}_{\cH} \mleft(\sqrt{k \cdot T} \mright)$, where $T$ is the time horizon. Unlearnable classes have (both randomized and deterministic) mistake bound $\Theta(T)$.

    
    Our upper bound is based on a deterministic algorithm for learning from expert advice with apple tasting feedback, a problem interesting in its own right. For this problem, we show that the optimal deterministic mistake bound is $\Theta \mleft(\sqrt{T (k + \log n)} \mright)$ for all $k$ and $T \leq n \leq 2^T$, where $n$ is the number of experts. Our algorithm is an optimal (up to constant factors), natural and computationally efficient variation of the well-known exponential weights forecaster.  
\end{abstract}

\pagebreak
\setcounter{tocdepth}{2}
\tableofcontents
\pagebreak

\section{Introduction}
We study online classification under the apple tasting feedback model presented in \cite{helmbold2000apple}. In this problem, the learner only observes the correct label when predicting $1$, in contrast to standard online learning where the learner observes the correct label after every prediction. In more detail, in the problem of online classification with apple tasting feedback, an \emph{adversary} and a \emph{learner} are rivals in a repeated game played for $T$ many rounds. In each round $t$, the adversary provides an instance $x_t$ from a domain $\cX$ and chooses a label $y_t \in \{0,1\}$, the learner then decides (possibly at random) on a prediction $\hat{y}_t \in \{0,1\}$ and suffers the loss $\ind[\hat{y}_t \neq y_t]$. If $\hat{y}_t = 1$, the adversary needs to send back the correct label $y_t \in \{0,1\}$ to the learner.

There are many examples for online learning scenarios in which the  provided feedback follows the apple tasting model. In many of them, the goal is to separate between good and bad instances, where the bad instances are trashed and their classification cannot be verified. A common example is the problem of \emph{spam filtering}. When a spam-filtering algorithm marks an e-mail as spam, it is moved to the spam folder and usually not presented to the user, which thus cannot verify that the e-mail is indeed spam. On the other hand, an e-mail marked by the algorithm as legitimate will go straight to the inbox, and the user can then indicate to the algorithm in case of a mistake, when the e-mail is spam.

Online classification with apple tasting feedback was previously studied by \cite{helmbold2000apple, raman2024apple}
However, the algorithms proposed in those works are randomized (except for some specific degenerate tasks). Consequently, prior to this work it was unknown if deterministic apple tasting is even possible. This is in contrast with other online learning settings, in which deterministic algorithms are usually studied, and considered as fundamental. Notable examples are the most basic $\SOA$ of \cite{littlestone1988learning} and the \emph{Weighted Majority} algorithm of \cite{littlestone1994weighted} for standard online classification, as well as the \emph{binoimal weights} algorithm for prediction with expert advice \cite{cesa1996line}. For online classification with bandit feedback, deterministic algorithms were presented in \cite{auer1999structural, daniely2015multiclass, long2020new}.

It is not a coincidence that in other online learning settings deterministic algorithms were designed and analyzed. There are many reasons to prefer deterministic (instead of randomized) learners, when possible:

\begin{enumerate}
    \item Mistake bounds of deterministic learners hold with probability $1$ and not in expectation.
    \item Randomness is a resource. Like any other resource, it could be missing or expensive.
    \item We do not have to worry about measurability constraints on the concept class we wish to learn.
    \item The predictions of deterministic learners can be simulated in advance, so every deterministic learner has the same guarantees for either an oblivious or an adaptive adversary.
    \item We do  not need to worry that the adversary might have knowledge on the internal mechanism generating the algorithm's randomness.
    \item The introduction of randomized algorithms can be seen as a necessity to overcome the impossibility of achieving sublinear regret by deterministic algorithms, as shown in \cite{cover1967behavior} by a simple argument. It is more natural to consider deterministic algorithms in realizable settings, where the argument of \cite{cover1967behavior} does not hold.
\end{enumerate}

This discussion raises the following natural main question, raised also by \cite{raman2024apple}, and guiding our work:

\begin{tcolorbox}
\begin{center}
    Is it possible to learn deterministically under apple tasting feedback?
\end{center}
\end{tcolorbox}

We answer this question positively: any class that can be learned in the easiest standard online learning setting, can be also learned deterministically, when only apple tasting feedback is provided. This affirmative answer raises the problem of finding the best possible deterministic mistake bounds achievable under apple tasting feedback. In Section~\ref{sec:results} we present
mistake bounds that are optimal in the sense that for some classes, they are the best possible (up to a constant factor). We also provide mistake bounds that are universally tight for every class, but only up to logarithmic factors of $T$ and constants depending on the class.

\subsection{Main results}\label{sec:results}
To state the results, we present some of the notation used in the paper. All notations are formally defined in Section~\ref{sec:preliminaries}. Let $\cH \in \{0,1\}^{\cX}$ be a \emph{concept class}, or a \emph{hypothesis class}, where $\cX$ is a \emph{domain} of instances. A pair $(x,y) \in \cX \times \{0,1\}$ is called an \emph{example}. 

\subsubsection{Learnability}
We say that $\cH$ is \emph{deterministically learnable under apple tasting feedback} (or \emph{deterministically learnable}, in short), if there exists a sub-linear function $M \colon \mathbb{N} \to \mathbb{N}$ and a deterministic learner that, for any $T \in \N$, $h^\star \in \cH$, and a \emph{realizable} input sequence of examples $S = (x_1, h^\star(x_1), \ldots, (x_T,h^\star(x_T))$, makes at most $M(T)$ prediction mistakes on $S$. The sequence $S$ is called realizable since its labels are realized by $h^\star \in \cH$. We also use the analogue natural definitions for randomized learnability with apple tasting feedback, and for learnability with full information feedback.\footnote{In the case of full-information feedback, the learnability definition is stronger and requires that there exists $M^\star \in \mathbb{N}$ so that $M(T) \leq M^\star$ for all $T$.}

We first wish to understand which classes are deterministically learnable. We show that even when receiving only apple tasting feedback and forcing the learner to be deterministic, the exact same classes which are learnable in the easiest standard online learning setting, are learnable in this difficult setting as well. 
\begin{theorem}[Qualitative characterization]
    The following conditions are equivalent for every class $\cH$.
    \begin{enumerate}
        \item $\cH$ is learnable with full-information feedback.
        \item $\cH$ is learnable with apple tasting feedback.
        \item $\cH$ is deterministically learnable with apple tasting feedback.
    \end{enumerate}
\end{theorem}
This characterization follows from our most basic quantitative upper bound on the mistake bound for learning a class $\cH$, which we will now discuss. For any class $\cH$ and time horizon $T$, let $\Mis(\cH, T)$ be the number of mistakes made by an optimal deterministic learner for $\cH$, on the worst-case realizable sequence $S$ of examples of length $T$. We sometimes refer to this quantity as the deterministic \emph{mistake bound} (or just mistake bound, when the context is clear) of $\cH$ with time horizon $T$. In his seminal work \cite{littlestone1988learning}, Littlestone proved that the optimal deterministic mistake bound of any class $\cH$, and for any time horizon $T$ in standard online learning (with full-information feedback) equals exactly to the \emph{Littlestone dimension} of $\cH$, denoted by $\LD(\cH)$. We prove the following mistake bounds in terms of $\LD(\cH)$ and $T$.

\begin{theorem}[Quantitative bounds - realizable case]{\label{thm:intro-concept-realizble}}
    For every class $\cH$ and time horizon $T$:
    \[
    \Mis(\cH,T) = O\mleft( \sqrt{\LD(\cH) T \log T} \mright).
    \]
    Furthermore, for every natural $d \geq 1$ there exists a class $\cH$ and $T_0(d)$ such that $\LD(\cH) = O(d)$ and
    \[
    \Mis(\cH,T) = \Omega\mleft( \sqrt{d T \log T} \mright)
    \]
    for all $T \geq T_0(d)$.
\end{theorem}

All of our upper bounds do not require that $T$ is given to the learner. The lower bound implies a strict separation between randomized and deterministic learners, since \cite{raman2024apple} showed that every  class can be learned with randomized mistake bound $O\mleft( \sqrt{\LD(\cH) T}  \mright)$.

\subsubsection{Agnostic learning}
It is interesting to prove a version of Theorem~\ref{thm:intro-concept-realizble} for the \emph{agnostic} setting, in which the input sequence $S$ can be any sequence of examples $(x_1, y_1), \ldots, (x_T,y_T)$, and not necessarily choose the labels $\{y_t\}_{t=1}^T$ to be $\{h^\star(x_t)\}_{t=1}^T$ for some $h^\star \in \cH$. In this setting, we measure $\Mis(\cH, T, k)$, which is the number of mistakes made by an optimal learner on the worst-case \emph{$k$-realizable} sequence $S$ of length $T$. The sequence $S = (x_1, y_1), \ldots, (x_T,y_T)$ is $k$-realizable if there exists $h^\star \in \cH$ so that $y_t \neq h^\star(x_t)$ for at most $k$ many indices $t$. We often call $k$ the \emph{realizability} parameter, as it indicates how far is $S$ from being realizable. All of our bounds do not require that $k$ is given to the learner. Theorem~\ref{thm:intro-concept-realizble} generalizes to the agnostic setting as follows.

\begin{theorem}[Quantitative bounds - agnostic case]{\label{thm:intro-concept-agnostic}}
    For every class $\cH$, time horizon $T$, and realizability parameter $k$:
    \[
    \Mis(\cH,T,k) = O\mleft( \sqrt{ T \mleft(k + \LD(\cH)\log T \mright)} \mright).
    \]
    Furthermore, for every natural $d \geq 1$ there exists a class $\cH$ and $T_0(d)$ such that $\LD(\cH) = O(d)$ and for every realizability parameter $k$ and $T \geq T_0(d)$:
    \[
    \Mis(\cH,T, k ) = \Omega\mleft( \sqrt{ T \mleft(k + d\log T \mright)} \mright).
    \]
\end{theorem}

Theorem~\ref{thm:intro-concept-agnostic} is implied by combining Theorem~\ref{thm:concept-upper-bound}, Theorem~\ref{thm:concept-realizable-lower-bound}, and Theorem~\ref{thm:concept-lowerbound}.

\subsubsection{Nearly-tight universal mistake bounds}
It is interesting to obtain mistake bounds that are nearly tight for all classes, and not just for the specific classes mentioned in the lower bound of Theorem~\ref{thm:intro-concept-agnostic}. When fixing the class $\cH$ and taking $T \to \infty$, we are able to obtain bounds which are tight up to logarithmic factors for every class. In \cite{raman2024apple}, a new combinatorial dimension of hypothesis classes coined \emph{effective width} is defined and proved to establish a  trichotomy in the randomized mistake bounds of hypothesis classes in the realizable setting. We recall the definition of effective width from \cite{raman2024apple} in the beginning of Section~\ref{sec:hypothesis-classes}.

The following bounds are proved in \cite{raman2024apple}. Let $\w(\cH)$ be the effective width of a class $\cH$, and let $\Mis_{\rand}(\cH, T)$,  $\Mis_{\rand}(\cH, T, k)$ be the analogue definitions of $\Mis(\cH, T)$,  $\Mis(\cH, T, k)$ when randomized learners are allowed. If $\w(\cH) = 1$ then $\Mis_{\rand}(\cH, T) =  \Theta_{\cH}(1)$, where the notation $\Theta_{\cH}(\cdot)$ hides constants depending on $\cH$. We call these \emph{easy} classes. Othrewise, if $1 < \w(\cH) < \infty$ then  $\Mis_{\rand}(\cH, T) =  \Theta_{\cH} \mleft( \sqrt{T} \mright)$. We call these \emph{hard} classes. Finally, if $\w(\cH) = \infty$ then  $\Mis_{\rand}(\cH, T) = \Theta(T)$. We call these \emph{unlearnable} classes. We prove that the same partition of all classes to easy, hard, and unlearnable classes by the  value of their effective width ($1$, finite larger than  $1$, or infinite) provides a trichotomy also for deterministic mistake bounds, and also for the agnostic setting, both for randomized and deterministic learners. However, our results demonstrate a large separation between deterministic and randomized mistake bounds of  hard classes in the agnostic setting, as  described in the following theorem.  

\begin{theorem}[Mistake bound trichotomy]\label{thm:intro-trichotomy}
    Fix a  class $\cH$. Let $k \in \mathbb{N}$ be the realizability parameter, and let $T \geq k$ be the time horizon. Then:
    \begin{enumerate}
        \item If $\cH$ is easy ($\w(\cH) = 1$) then:
        \[
        \M^\star_{\rand}(\cH, T, k), \M^\star(\cH, T, k) = \Theta_{\cH}(k).
        \]
        \item If $\cH$ is hard ($1 < \w(\cH) < \infty$) then:
        \[
        \M^\star_{\rand}(\cH, T, k) = \Theta(k) + \tilde{\Theta}_{\cH} \mleft(\sqrt{T} \mright), \quad  \M^\star(\cH, T, k) = \Theta \mleft(\sqrt{T k} \mright) + \tilde{\Theta}_{\cH} \mleft(\sqrt{T} \mright).
        \]
        \item If $\cH$ is unlearnable ($\w(\cH) = \infty$) then:
        \[
        \M^\star_{\rand}(\cH, T, k), \M^\star(\cH, T, k) = \Theta(T).
        \]
    \end{enumerate}
\end{theorem}
We prove Theorem~\ref{thm:intro-trichotomy} in Section~\ref{sec:trichotomy}. Theorem~\ref{thm:intro-trichotomy} implies a large seperation between randomized and deterministic learning of hard classes: for any hard class $\cH$, if $k$ grows for example as  $\sqrt{T}$, then the randomized mistake bound grows at most as $\sqrt{T \log T}$, while the deterministic mistake bound grows as $T^{3/4}$.

\subsubsection{Prediction with expert advice}

One of the main building blocks in this paper, which is interesting on its own right, is a deterministic algorithm for \emph{prediction with expert advice} with apple tasting feedback. Our algorithm is the first deterministic learner for this problem, and it is optimal (up to constant factors) for reasonably large enough number of experts $n$. It uses a novel variation of the well-known exponential weights forecaster (see \cite[Section 2.2]{cesa2006prediction}). This problem is similar to the problem of learning hypothesis classes, only that instead of hypotheses, having predictions that are determined by the provided instance $x_t$, we have experts that decide on their predictions in every round, possibly in an adversarial manner. For this problem, we prove the following bounds. Let $\M^\star(n,T,k)$ be the optimal mistake bound in the instance of the problem where there are $n$ experts, the time horizon is $T$ and the best expert makes at most $k$ many mistakes. 

\begin{theorem}[Optimal bound for experts]\label{thm:intro-experts}
    For every $n, T, k$ so that $k \leq T \leq n < 2^{T/2}$ and $T \geq 2$:
    \[
    \M^\star(n,T,k) = \Theta \mleft( \sqrt{T (k + \log n)} \mright).
    \]
\end{theorem}

We prove Theorem~\ref{thm:intro-experts} in Section~\ref{sec:experts-agnostic} (upper bound) and Section~\ref{sec:experts-lower-bounds} (lower bound).

Similar results for randomized learners were proved in \cite{helmbold2000apple,raman2024apple}. \cite{raman2024apple} proved a randomized upper bound of $k + O\mleft(\sqrt{T \log n} \mright)$. 
\cite{helmbold2000apple} studied the realizable case with an oblivious adversary (all expert's predictions are fixed in advance) and showed that the optimal randomized mistake bound in this setting is
$
\Theta\mleft( \sqrt{\frac{T \log n}{\log \frac{T}{\log n}}} \mright)
$
in the range of parameters of Theorem~\ref{thm:intro-experts} (in their work, they proved tight bounds for all regimes of $T,n$).

Following these results, it is interesting to observe that in the experts' setting, a strict separation between randomized and deterministic learners in the agnostic case appears via a simple argument, given that $n$ is sufficiently large. Consider an adversary that operates in $\sqrt{T/k}$ many phases, each of $\sqrt{Tk}$ many rounds, where in phase $i$ the only expert predicting $1$ is $i$, and the correct label reported by the adversary is always $0$. It is not difficult to prove that $\sqrt{Tk}$ many mistakes can be forced on a deterministic learner regardless of its predictions, which is impossible for an optimal randomized learner. A formal statement and proof are given in Theorem~\ref{thm:experts-agnostic-lowerbound}.

\subsection{Additional related work}
In this section, we mention some additional related work not mentioned in Section~\ref{sec:results}.

Apple tasting feedback, like other partial feedback models, can be seen as a special case of the more general graph-feedback model \cite{mannor2011bandits, alon2015online}, which in turn is a special case of partial monitoring \cite{bartok2014partial}. The more famous bandit-feedback model \cite{daniely2015multiclass, daniely2013price} is also a special case of the graph-feedback model.

This work aims to understand how the restriction of using only deterministic learners affects the mistake bounds in the apple-tasting feedback model. Problems in the  same vein were studied for other feedback models: full-information feedback in \cite{abernethy2006continuous, filmus2023optimal}, and bandit-feedback in \cite{filmus2024bandit}.

Most of our upper bounds are based on a novel algorithm we present for the problem of \emph{prediction with expert advice}, which is a variation of the \emph{exponentially weighted average forecaster} of \cite{cesa2006prediction}. Using this algorithm, we prove tight bounds that hold under the restriction that both the experts' and learner's predictions are deterministic, and only apple-tasting feedback is provided. Many variations of this problem were extensively studied in the past decades. The most basic setting of binary online learning with full-information feedback was studied, e.g,\ in \cite{vovk1990aggregating, littlestone1994weighted, cesa1997use, abernethy2006continuous, filmus2023optimal}. The work of \cite{cesa1996line} studied this setting with deterministic predictions. In \cite{branzei2019online}, the multiclass version of the problem is studied. The seminal work of \cite{auer2002nonstochastic} studied the bandit-feedback variation of the problem.

\section{Technical overview}
In this section, we informally describe the ideas behind the main novel techniques used to prove our results. We focus on  the realizable setting, which already captures many of the ideas we find intereseting.

\paragraph{Section organization.} In Section~\ref{sec:tech-experts} we overview our algorithm for the problem of prediction with expert advice, which attains the upper bound of Theorem~\ref{thm:intro-experts} in the realizable case $(k=0)$. This result is a main building block of the results proved in this paper. We then explain in Section~\ref{sec:tech-hypothesis} how to use it to prove the upper bound of Theorem~\ref{thm:intro-concept-realizble}. Finally, in Section~\ref{sec:tech-hypothesis-lowerbound} we explain how to construct the classes mentioned in the lower bound  of Theorem~\ref{thm:intro-concept-realizble}.

\subsection{Technical overview: upper bound for prediction with expert advice} \label{sec:tech-experts}
Recall that in the realizable setting of prediction with expert advice with apple tasting feedback, we have $n$ many experts, where each expert $i$ provides a prediction $\hat{y}^{(i)}_t \in \{0,1\}$ on each round $t$,  such that there exists $i^\star \in [n]$ for which $\hat{y}^{(i^\star)}_t = y_t$ for all $t$ (where $y_t$ is the correct label). The learner's goal, as usual, is to minimize the number of rounds where it errs. The realizability assumption implies that whenever the learner predicts $\hat{y}_t = 1$ but $y_t = 0$, it can eliminate all experts predicting $1$ in this round. Note that these rounds are exactly the false positives, and denote their number with $M_{+}$. Let $M_- = |t \in [T]: \hat{y}_t = 0, \hat{y}^{(i^\star)}_t = 1|$, which is exactly the number of false negatives.  The total number of mistake $M$ made by the learner is thus exactly $M = M_+ + M_-$. Therefore, the learner wishes to balance $M_+$ and $M_-$, such that both are not too high.

Consider the following two extreme scenarios, which are very easily handled by the learner:
\begin{enumerate}
    \item In every round, the adversary let exactly $a$ experts to predict $1$, where $a \geq 1$ is some constant.
    \item In every round, the adversary let exactly $b \cdot n$ experts to predict $1$, where $b < 1$ is some constant.
\end{enumerate}

We call the experts that are  not yet eliminated the \emph{living} experts.
In the first scenario, the learner can predict $1$ only when some living expert $i$ has $|t \in [T]: \hat{y}^{(i)}_t = 1, \hat{y}_t = 0| \geq \sqrt{T}$. This strategy is similar to the deterministic learner of \cite{helmbold2000apple} for the class of singletons, and produces $\M_+, M_- = O \mleft( \sqrt{T} \mright)$.
In the second scenario, the learner can predict $1$ in every round, which produce $M_- = 0$ and $M_+ = O(\log n)$.

Note that in the first scenario, the adversary always let very few experts to predict $1$. It thus wastes many rounds on increasing the value of $|t \in [T]: \hat{y}^{(i)}_t = 1, \hat{y}_t =  0|$, which reflects $M_-$ in case that $i$ realizes the input sequence, for just a small number of experts. Therefore, in this scenario the learner benefits from making most of its predictions $0$. The second scenario is the other extreme, in which the adversary allows the learner to eliminate many experts in every round in which it predicts $1$ and make a mistake, and thus the learner benefits from always predicting $1$. The same idea of course applies even if the adversary follows the first scenario in some rounds, and the second scenario in other rounds. Unfortunately, the adversary is not limited to those scenarios, and can follow many other intermediate scenarios, which are harder to handle. Roughly speaking, to handle it we need to ``smooth" the binary treatment of both extreme scenarios to some unified strategy. This is done by using \emph{exponential weights}. For every expert $i$, we denote $d^{(i)}_t = |t \in [T]: \hat{y}^{(i)}_t = 1, \hat{y}_t = 0|$, and call $d^{(i)}_t$ the \emph{distance} of expert $i$ in round $t$. The distance of expert $i$ reflects the number of  false negatives to be made by the learner in case that the input sequence is realized by expert $i$. We define $2^{\eta d^{(i)}_t}$ to be the \emph{weight} of expert $i$ in round $t$, where $\eta$ is a parameter to be chosen optimally. In every round $t$, our algorithm predicts $1$ if and only if the sum of weights of living experts predicting $1$ is larger than $n$. Namely, the algorithm predicts $1$ if and only if:

\begin{equation} \label{eq:condition}
    \sum_{i \in V_t} \ind[\hat{y}^{(i)}_t = 1] 2^{\eta d^{(i)}_t} \geq n,   
\end{equation}

where $V_t$ is the set of living experts in round $t$. Note that this strategy handles the two extreme scenarios well: If in round $t$ just a few living experts predict $1$ and the learner predicts $1$ as well, then those experts must have already travelled a long distance, which means that the adversary has wasted a lot of rounds on increasing their distance. On the other hand, if many experts predict $1$, then even if all of them have low distances, condition~\eqref{eq:condition} is satisfied. Using an exponential value of the distance as weight in condition~\eqref{eq:condition} produces a balanced strategy that can handle all adversarial strategies. We formalize this in the proof the upper bound of Theorem~\ref{thm:intro-experts} for the realizable case ($k=0$), in Section~\ref{sec:experts-realizable}. In  a nutshell, note that $M_- \leq \frac{\log n}{\eta}$, since any expert can reach distance at most $\frac{\log n}{\eta}$ by condition~\eqref{eq:condition}. To bound $M_+$, we show that in order to make the learner predict $1$, the adversary must either waste many rounds, or many experts to be eliminated if the learner predicts $1$. The bound we get is $M_+ = O(\eta T)$. Thus, we get in total
\begin{equation}\label{eq:similar-to-agnostic-regret}
    M = O \mleft(\frac{\log n}{\eta} + \eta T \mright).
\end{equation}

Choosing $\eta = \sqrt{\frac{\log n}{T}}$ gives the upper bound $O \mleft( \sqrt{T \log n} \mright)$ of Theorem~\ref{thm:intro-experts} when $k=0$. Remarkably, the bound~\eqref{eq:similar-to-agnostic-regret} is exactly the regret bound for agnostic prediction with expert advice under full information feedback appearing in \cite[Corollary 2.2]{cesa2006prediction}. This does not seem like a coincidence: our algorithm is a variation of their \emph{exponentially weighted average forecaster} achieving the regret bound of their Corollary~2.2. However, we do not know if there is a direct rigorous connection between detereministic realizable prediction with expert advice under apple tasting feedback to agnostic prediction with expert advice under full-information feedback (for example, is there a reduction between the problems in some direction?).

\subsection{Technical overview: upper bound for learning hypothesis classes} \label{sec:tech-hypothesis}
To prove the upper bound of Theorem~\ref{thm:intro-concept-realizble}, we use a reduction from learning hypothesis classes to the problem of prediction with expert advice, proved and used in \cite{bendavid2009agnostic} for agnostic learning with full-information feedback. They showed that for every class $\cH$ and every $T$, there exists a single class of $n = T^{\LD(\cH)}$ many experts that \emph{covers} any sequence of instances $\{x_t\}_{t=1}^T$  of length $T$ with respect to $\cH$, where ``covers" means that for every $h \in \cH$, there exists $i\in[n]$ so that $\hat{y}^{(i)}_t = h(x_t)$ for every $t \in [T]$. This means that there is essentially no difference between learning $\cH$, and predicting using the advice of the covering experts. Applying the bound of Theorem~\ref{thm:intro-experts} with $n= T^{\LD(\cH)}$ gives the upper bound in Theorem~\ref{thm:intro-concept-realizble}. 

\subsection{Technical overview: lower bound for specific hypothesis classes}\label{sec:tech-hypothesis-lowerbound}
The lower bound of Theorem~\ref{thm:intro-concept-realizble} claims that for every natural $d$ larger  than some universal constant there exists $T_0(d)$ and a class $\cH$ such that $\LD(\cH) = O(d)$ and $\Mis(\cH,T) = \Omega \mleft( \sqrt{d T \log T} \mright)$ for all $T \geq T_0(d)$.
The class of $n$ experts can be seen as a concept class $\cH$ of size $n$, and it has $\LD(\cH) = \floor*{\log n}$. Therefore, using this class directly with the lower bound of Theorem~\ref{thm:intro-experts} will only give a lower bound of $\Omega\mleft( \sqrt{\LD(\cH) T}  \mright)$. However, looking into the proof of the lower bound in Theorem~\ref{thm:intro-experts} reveals that in contrast with full-information  lower bounds, the ``hard" experts' predictions used by the adversary are very unbalanced, in the sense that many experts predict $0$ and only few predict $1$. Since $\LD(\cH) = \Omega(\log n)$ comes from choosing experts' predictions which are as balanced as possible, this intuitively means that we can restrict the adversary such that only unbalanced experts' predictions are available, in a way that decreases the Littlestone dimension from $\Omega(\log n)$ to $O(\log_T n)$, but maintains the $\Omega \mleft(\sqrt{T \log n} \mright)$ lower bound from Theorem~\ref{thm:intro-experts}. Choosing $n = T^d$ results in a class $\cH$ with $\LD(\cH) = O(d)$ that maintains a lower bound of $\Omega(T \log T^d)$. In Section~\ref{sec:concept-specific-lowerbound}, we provide a non-constructive proof of existence of $\cH$. Finding an explicit class $\cH$ realizing this lower bound remains open.

The technique explained above inherently assumes that $T$ is given in advance, before choosing the concept class used for the lower bound. As this is not the case, we still need to remove this assumption.
This is done by ``gluing" together instances of the class described above for all $T \geq T_0(d)$, in a way that keeps the Littlestone dimension $O(d)$. The details of how to glue the classes can be found in Section~\ref{sec:concept-specific-lowerbound}.

\section{Preliminaries} \label{sec:preliminaries}

\subsection{Realizable learning definitions}

Let $\cX$ be a (possibly infinite) \emph{domain}.
A pair $(x, y)\in \cX \times \{0,1\}$ is called an \emph{example}, and an element $x\in \cX$ is called an \emph{instance}. 
A function $h\colon \cX\to\{0,1\}$ is called a \emph{hypothesis} or a \emph{concept}.
A \emph{hypothesis class}, or \emph{concept class}, is a 
non-empty  set $\cH \subset \{0,1\}^{\cX}$.
A sequence of examples  $S=\{(x_i,y_i)\}_{t=1}^T$ is said to be \emph{realizable} by $\cH$ if there exists $h \in \cH$ such that $h(x_t) = y_t$ for all $1 \leq i \leq T$. We say that such $h$ is \emph{consistent} with $S$, or \emph{realizes} it. For simplicity of results statements, we will always assume that $n,T \geq 2$. The notation $\ind[\cdot]$ denotes an indicator function.
Online learning with apple tasting feedback \cite{helmbold2000apple} is a repeated game between a learner and an adversary.
Each round $t$, in the game proceeds as follows:
\begin{enumerate}[(i)]
\item The adversary picks an example $(x_t, y_t) \in \cX \times \{0,1\}$, and reveals only $x_t$ to the learner.
\item The learner predicts a value $\hat{y}_t \in \{0,1\}$ and suffers the \emph{loss} $\ind[\hat{y}_t \neq y_t]$.  
\item If $\hat{y}_t = 1$, the adversary reveals $y_t$ to the learner.
\end{enumerate}

In this work, unless stated otherwise, we study the case where the predictions of the learner are deterministic.
We model apple tasting learners as functions $\Lrn \colon (\cX \times \{0,1,\star\})\strut^* \times \cX \rightarrow \{0,1\}$. The input of the learner has two parts: a \emph{feedback sequence} $F \in (\cX \times \{0,1,\star\})\strut^*$, and the current instance $x \in \cX$. The feedback sequence is naturally constructed throughout the game: if in round $t$ the prediction is $0$, then the learner appends $(x_t, \star)$ to the feedback sequence, to indicate that no feedback was given for $x_t$. If, on the other hand, the prediction is $1$, then the learner appends $(x_t,y_t)$ to the feedback sequence. In round $t$, the prediction of the learner is $\Lrn(F, x_t)$, where $F$ is the feedback sequence gathered by the learner in rounds $1, \dots, t-1$.

Given a learning rule $\Lrn$ and an input sequence of examples $S = (x_1,y_1),\ldots,(x_T,y_T)$ such that $(x_t,y_t)$ is the example picked by the adversary in round $t$, we denote the number of mistakes that $\Lrn$ makes on $S$ by
\[\M(\Lrn; S) = \sum_{i=1}^T  \ind [\hat{y}_t \neq y_t].\]
It is worth noting that fixing $S$ beforehand is usually linked with an \emph{oblivious} adversary setting, in which the adversary cannot pick the examples on the fly. However, when the learner is deterministic, the adversary can simulate the entire game on its own, since we assume that the learning algorithm is known to all. 
Thus, oblivious and adaptive adversaries are in fact equivalent, ans we will refer to the adversary as being either adaptive or oblivious, depending on whichever is more convenient in the given context.

Let $\mathbb{R}^+$ be the set of non-negative real numbers. 
An hypothesis class $\cH$ is \emph{learnable with apple tasting feedback}, if there exists a function $M_{\cH} \colon \mathbb{R}^+ \to \mathbb{R}^+$ satisfying $\lim_{T \to \infty}\frac{M_{\cH}(T)}{T} = 0$, and a learning rule $\Lrn$, such that for any $T$ and for any input sequence $S$ of length $T$ which is realizable by $\cH$, it holds that $\M(\Lrn; S) \leq M_{\cH}(T)$.

We define the \emph{optimal} mistake bound of $\cH$ with horizon $T$ to be
\begin{equation}
\M^{\star}(\cH,T)=\adjustlimits\inf_ {\Lrn} \sup_{S} \M(\Lrn ;S),
\end{equation}
where the infimum is taken over all deterministic learning rules, and the supremum is taken over all realizable input sequences $S$ of length $T$. It is convenient to assume that $T$ is given to the learner. However, in Section~\ref{sec:without-prior} we show how to remove this assumption using standard doubling tricks. 

\subsection{Agnostic learning definitions}
We also study the \emph{agnostic} setting, in which even the best hypothesis might be inconsistent with the input sequence. Towards this end, we use the \emph{$k$-realizable} framework, in which it is assumed that an upper bound $k$ on the number of mistakes made by the best hypothesis is given to the learner. This assumption is convenient but not necessary, and we show in Section~\ref{sec:without-prior} how to remove it. Formally, an input sequence $S = (x_1,y_1),\ldots,(x_T,y_t)$ is $k$-realizable by a class $\cH$ if there exists $h \in \cH$ such that $h(x_t) \neq y_t$ for at most $k$ indices $t \in [T]$. We say that such a hypothesis is \emph{$k$-consistent} with $S$, or \emph{$k$-realizes} it. For a given $k$, we denote the corresponding optimal mistake bound of $\cH$ with horizon $T$ by
\begin{equation}
\M^{\star}(\cH, T, k)=\adjustlimits\inf_ {\Lrn} \sup_{S} \M(\Lrn ;S),
\end{equation}
where the infimum is taken over all deterministic learning rules, and the supremum is taken over all $k$-realizable input sequences $S$ of length $T$. Note that the realizable setting is a special case of the $k$-realizable setting, attained when $k=0$. We define and sometimes analyze the realizable case separately for didactic reasons: it is often significantly easier to handle.

\subsection{Decision trees and the Littlestone dimension.}
In this paper, a \emph{tree} $\tree$ refers to a finite full rooted ordered binary tree (that is, a rooted binary tree where each node which is not a leaf has a left child and a right child), equipped with the following information:
\begin{enumerate}
    \item Each internal node $v$ is associated with an instance $x \in \cX$.
    \item For every internal node $v$, the left outgoing edge is associated with the label $0$, and the right outgoing edge is associated with the label $1$.
\end{enumerate}

A \emph{prefix} of the tree $T$ is any path that starts at the root. In this paper, a path is defined by a sequence of consecutive vertices. If a path is not empty, we may refer it by the sequence of consecutive edges corresponding with the sequence of consecutive vertices defining it. 
    A prefix $v_0,v_1, \dots, v_t$ defines a sequence of examples $(x_1,y_1), \dots, (x_{t},y_{t})$ in a natural way: 
    for every $i \in [t]$, $x_i$ is the instance corresponding to the node $v_{i-1}$, 
    and $y_i$ is the label corresponding to the edge $v_{i-1}\to v_{i}$. 
    A prefix is called \emph{maximal} if it is maximal with respect to containment, 
    that is, there is no prefix in the tree that strictly contains it. 
    This is equivalent to requiring that $v_{t}$ be a leaf. 
    A maximal prefix is called a \emph{branch}. 
    The length of a prefix is the number of edges in it (so, the length is equal to the size of the corresponding sequence of examples). The length of a maximal branch in a tree is referred to as the tree's \emph{depth}. We sometimes also refer to the length of a prefix as its depth.

A prefix in the tree is said to be \emph{$k$-realizable} by $\cH$ if the corresponding sequence of examples is $k$-realizable by $\cH$.
    A tree $\tree$ is \emph{$k$-shattered} by $\cH$ if all branches in $\tree$ are $k$-realizable by $\cH$. For every branch, it is convenient to relate to some hypothesis $k$-realizing it as the hypothesis that labels the leaf at the end of this branch.
    The \emph{Littlestone dimension} of an hypothesis class $\cH$, denoted by $\LD(\cH)$,  is the maximal depth of a \emph{perfect}  tree (that is, a tree in which all branches have the same depth) shattered by $\cH$. 
    If $\cH$ shatters trees of arbitrarily large depth, then $\LD(\cH)=\infty$.

\section{Prediction with expert advice}
In this section we study the optimal mistake bound of a specific class called the \emph{universal class} (of size $n$). For every $n \geq 1$, the universal class $\cU_n = \{h_1, \ldots, h_n\}$ is defined as follows. Let $\cX_n = \{0,1\}^n$ be the instance domain. For every $i \in [n]$, the hypothesis $h_i$ is defined to be $h_i(x) = x_i$ for every $x \in \cX_n$, where $x_i$ denotes the $i$'th entry of $x$.

Note that every partition of $\cU_n$ to hypotheses predicting $0$ and hypotheses predicting $1$ has an appropriate $x \in \cX_n$ inducing it. Therefore, learning $\cU_n$ is equivalent to the game of \emph{prediction with expert advice with apple tasting feedback}, when the experts provide deterministic predictions. Formally, each round $t$ of the game proceeds as follows:

\begin{enumerate}[(i)]
\item The $n$ experts present predictions $\hat{y}_t^{(1)},\ldots,\hat{y}_t^{(n)} \in \{0,1\}$.
\item The learner predicts a value $\hat{y}_t \in \{0,1\}$ and suffers the \emph{loss} $\ind[\hat{y}_t \neq y_t]$.  
\item If $\hat{y}_t = 1$, the adversary reveals $y_t$ to the learner.
\end{enumerate}

In the realizable setting, there is an expert who never errs, and in the $k$-realizable setting there is an expert who errs for at most $k \in \mathbb{N}$ many times. The upshot of viewing $\cU_n$ as a class of experts is twofold: first, it makes the problem easier to describe and analyze. Second, it allows to pick the experts to be learning algorithms and extend the solution to all concept classes, as done in Section~\ref{sec:hypothesis-classes}. Denote

\[
\M^{\star}(n, T) = \M^{\star}(\cU_n, T), \quad \text{and} \quad \M^{\star}(n, T, k) = \M^{\star}(\cU_n, T, k).
\]

Note that $\M^{\star}(n, T) = \M^{\star}(n, T, 0)$.
For brevity, we will usually refer to the experts simply by their indices $[n]$.
The main goal of this section is to prove an upper bound on $\M^{\star}(n, T, k)$ for all non-trivial $n, T, k$. Since the solution for $k=0$ is simpler, to be more didactic we will start with $k=0$ and then handle the general case.

\subsection{Prediction with expert advice: realizable case} \label{sec:experts-realizable}
In this section we prove the following upper bound on $\M^{\star}(n, T)$, attained by the $\RealizableExpAT$ learner  presented in Figure~\ref{fig:RealizableExpAT}.

\begin{theorem} \label{thm:experts-realizable-upper-bound}
    For every time horizon $T \geq 2$, number of experts $n  \geq 2$, and a realizable input sequence $S$ of experts' predictions and true labels of length $T$:
    \[
    \M(\RealizableExpAT, S) = O \mleft( \sqrt{T \log n} \mright).
    \]
    Furthermore, $\RealizableExpAT$ is computationally efficient.
\end{theorem}

In section~\ref{sec:experts-lower-bounds} we show that the upper bound in Theorem~\ref{thm:experts-realizable-upper-bound} is tight for $n \geq T$.  Our algorithm is a variation of the known exponential weights forecaster \cite[Section 2.2]{cesa2006prediction}. Let us briefly describe $\RealizableExpAT$. We maintain a version space $V_t$ containing of all experts that have been consistent with the feedback gathered until round $t-1$. Initially, $V_1$ contains all the experts. For each expert $j \in V_t$, we maintain a weight which will be a function of its \emph{distance} $d^{(j)}_t$. The distance of an expert in round $t$ is defined to be the number of rounds in which it has predicted $1$ while the learner predicted $0$, until round $t-1$. We define the \emph{weight} of expert $j$ in round $t$ to be $2^{\eta d^{(j)}_t}$, where $\eta$ is a parameter to  be optimized. The algorithm's decision making mechanism is very simple and efficient: it predicts $1$ if and only if the total sum of weights of all experts in $V_t$ predicting $1$ exceeds some threshold $L$, where $L$ is a parameter to be optimized. If the learner has made a mistake when predicting $1$, it removes all experts predicting $1$ from the version space.

\begin{figure}
    \centering
    \begin{tcolorbox}
    \begin{center}
        \textsc{$\RealizableExpAT$}
    \end{center}
    \textbf{Input:} Set of experts indexed by $[n]$, learning parameters $\eta \in (0,1), L > 1$.
    \\
    \textbf{Initialize:} Let $V^{(1)} = [n]$, let $d^{(j)}_1 = 0$ for all $j \in [n]$.\\
    \\
    \textbf{for $t=1,\ldots, T$:}
    \begin{enumerate}
        \item Receive experts predictions $\hat{y}^{(1)}_t, \ldots, \hat{y}^{(n)}_t$.
        \item Predict $\hat{y}_t = 1$ if and only if
        \[
        \sum_{j \in V_t} \ind[\hat{y}^{(j)}_t = 1] 2^{\eta d^{(j)}_t} \geq L.
        \]
        \item If $\hat{y}_t = 1$:
        \begin{enumerate}
            \item Set $d^{(j)}_{t+1} = d^{(j)}_{t}$ for every expert $j\in V_t$.
            \item If $y_t = 0$: set $V_{t+1} = V_t \backslash \{j \in V_t: \hat{y}^{(j)}_t = 1\}$.
        \end{enumerate}
        \item If $\hat{y}_t = 0$: set $d^{(j)}_{t+1} = d^{(j)}_{t} + \ind[\hat{y}^{(j)}_t = 1]$ for every expert $j\in V_t$.
    \end{enumerate}
    \end{tcolorbox}
    \caption{A deterministic apple tasting learner for  realizable prediction with expert advice.} 
    \label{fig:RealizableExpAT}
\end{figure}

We will now analyze the algorithm's mistake bound. Towards this end, we introduce some notation.
Fix the number of experts $n \geq 2$, the horizon $T \geq 2$, and an execution of $\RealizableExpAT$ on a sequence $S$ of experts' predictions and true labels.  A mistake in which the algorithm predicts $1$ is called a \emph{false positive}, and a mistake in which the algorithm predicts $0$ is called \emph{false negative}. For $r \in \{0,1\}$, let $T_r$ be the set of rounds in which $\RealizableExpAT$ predicts $r$. Let $T_- = \{t\in T_0: y_t = 1\}$ be the set of rounds in which $\RealizableExpAT$ makes a false negative mistake, and let $T_+ = \{t\in T_1: y_t = 0\}$ be the set of rounds in which $\RealizableExpAT$ makes a false positive mistake. Denote $M_- = |T_-|$, $M_+ = |T_+|$. For $d\in \mathbb{N}$, let $I_d = [d/\eta, (d+1)/\eta]$. For every expert $j$, let $D_j$ be it's distance in the last round where it is still in the version space, or at round $T$ if $j \in V_T$. Let $N_d = \{j : D_j \in I_d\}$ and $n_d = |N_d| $.
    
\begin{lemma} \label{lem:experts-realizable-main}
    For any learning parameters $\eta \in (0,1), L > 1$ satisfying $n \leq \eta L T$, we have:
    \[
    M_- \leq \frac{\log L}{\eta} + 1, \quad \text{and} \quad M_+ \leq 6 \eta T.
    \]
\end{lemma}

The challenging part in proving Lemma~\ref{lem:experts-realizable-main} is to upper bound $M_+$. The key ingredient of the proof of this bound is an inequality proved in Lemma~\ref{lem:experts-realizable-helper} by forming lower and upper bounds on the \emph{total weighted distance gained} by experts in the version space. Formally, we will lower and upper bound the quantity
\[
\TWDG := \sum_{t \in T_0} \sum_{j \in V_t} \ind[\hat{y}^{(j)}_t = 1] 2^{\eta d^{(j)}_t}.
\]
Indeed, by multiplying the weight of expert $j$ in round $t \in T_0$ by $\ind[\hat{y}^{(j)}_t = 1]$, we make sure to take into account its current weight only in rounds where it gains distance. We sum this quantity over all rounds and all experts to get $\TWDG$.

\begin{lemma} \label{lem:experts-realizable-helper}
    Suppose that $n \leq \eta T L$. Then:
    \[
    \sum_{d \in \mathbb{N}}n_d  2^d \leq 3 \eta TL.
    \]
\end{lemma}

\begin{proof}
    The main part of the proof is to lower and upper bound $\TWDG$. We start with the upper bound. By  definition of the algorithm, for every $t \in T_0$ it holds that
    \[
    \sum_{j \in V_t} \ind[\hat{y}^{(j)}_t = 1] 2^{\eta d^{(j)}_t} < L,
    \]
    which implies $\TWDG \leq L T_0 \leq LT$.
    For the lower bound, note that for every $d \geq 1$, every expert $j$ with $D_j \in I_d$ must have gained distance for $1/\eta$ many times while having weight at least $2^{\eta (d-1)/\eta} = 2^{d-1}$. Therefore:
    \[
    \TWDG \geq \sum_{d \geq 1} n_d \frac{1}{\eta}  2^{d-1}. 
    \]
    Combined with the upper bound $\TWDG \leq LT$, this implies
    \[
    \sum_{d \geq 1} n_d  2^d \leq 2 \eta TL.
    \]
    Applying the assumption $n \leq \eta T L$ gives the statement of the lemma.
\end{proof}

We can now prove Lemma~\ref{lem:experts-realizable-main}.

\begin{proof}[Proof of Lemma~\ref{lem:experts-realizable-main}]
    Let's first bound $M_-$. Let $j^\star$ be an expert who makes no mistakes.
    In each of the rounds of $T_-$, the distance of $j^\star$ is increased by $1$. Therefore, if $|T_-| > \frac{\log L}{\eta} + 1$, then  $d_t^{(j^\star)} > \frac{\log L}{\eta} + 1$ for some round $t$.
    However, by the algorithm's definition, this means that there was some round $t > t' \in T_0$ for which $d_{t'}^{(j^\star)} > \log L/ \eta$ and $\hat{y}^{(j^\star)}_{t'} = 1$. This is not possible, since it would have imply $2^{\eta d_{t'}^{(j^\star)}} > L$ while $j^\star$ predicts $1$, which implies $\hat{y}_{t'} = 1$.

    Let's now bound $M_+$. Denote $m  =  M_+$. Denote $T_+ = \{t_1, \ldots, t_m\}$. For every round $t_i \in T_+$, the set $A_i$ of experts predicting $1$ must have sum of weights at least $L$. Therefore, the set $A_i$ will be removed from the version space at the end of round $t_i$, which implies that all $A_i$ are disjoint. Therefore, we have:
    \[
    \sum_{i \in [m]} \sum_{j \in A_i} 2^{\eta D_j} \geq m L.
    \]
    On the other hand, the same quantity is upper bounded by:
    \[
    \sum_{i \in [m]} \sum_{j \in A_i} 2^{\eta D_j}
    = \sum_{d \in \mathbb{N}} \sum_{i \in [m]} \sum_{A_i \cap \{j: D_j \in I_d\}} 2^{\eta D_j}
    \leq \sum_{d \in \mathbb{N}} n_d 2^{d + 1} \leq 6 \eta TL,
    \]
    where the last inequality is due to Lemma~\ref{lem:experts-realizable-helper}. Combining both equations above gives $m \leq 6 \eta T$ as required.
\end{proof}

It is now straightforward to infer Theorem~\ref{thm:experts-realizable-upper-bound}.

\begin{proof}[Proof of Theorem~\ref{thm:experts-realizable-upper-bound}]
    We run $\RealizableExpAT$ with learning parameters $L = n$ and $\eta = \sqrt{\frac{\log n}{T}}$. The parameters are valid since $L > 1$ and $\eta \in (0,1)$. Further, we have $\eta L T = \sqrt{\frac{\log n}{T}} n T \geq n$, since $\sqrt{\frac{\log n}{T}} T = \sqrt{T \log n} \geq 1$. Therefore Lemma~\ref{lem:experts-realizable-main} holds and we have
    \[
    \M(\RealizableExpAT, S) = M_- + M_+ \leq \frac{\log L}{\eta} + 1 + 6 \eta T = \sqrt{T \log n} + 1 + 6\sqrt{T \log n}.
    \]
    Furthermore, it is clear that $\RealizableExpAT$ is computationally efficient by its definition.
\end{proof}

\subsection{Prediction with expert advice: agnostic case} \label{sec:experts-agnostic}

In this section we extend Theorem \ref{thm:experts-realizable-upper-bound} to the $k$-realizable setting, in which an upper bound $k \in \mathbb{N}$ on the number of mistakes of the best expert is given. This assumption is not necessary, and in Section~\ref{sec:without-prior} we show how to remove it using a variation of the standard ``doubling trick" of, e.g., \cite{cesa1997use}. This section follows lines similar to the previous section that handles the realizable case. We prove the following upper bound, attained by the $\ExpAT$ learner presented in Figure~\ref{fig:ExpAT}.

\begin{theorem} \label{thm:experts-agnostic-upper-bound}
    For every number of experts $n \geq 2$, horizon $T \geq 2$, realizability parameter $k$, and a $k$-realizable input sequence $S$ of experts' predictions and true labels of length $T$:
    \[
    \M(\ExpAT, S) = O \mleft( \sqrt{T (k + \log n)} \mright).
    \]
    Furthermore, $\ExpAT$ is computationally efficient.
\end{theorem}

In section~\ref{sec:experts-lower-bounds} we show that the upper bound in Theorem~\ref{thm:experts-agnostic-upper-bound} is tight for $n \geq T$. $\ExpAT$ is almost identical to $\RealizableExpAT$. The difference is that in $\RealizableExpAT$ the weight of an expert $j$ in round $t$ is defined to be $2^{\eta d_t^{(j)}}$, and in $\ExpAT$ it is defined to be $2^{k_t^{(j)} + \eta d_t^{(j)}}$, where $k_t^{(j)}$ is the number of future \emph{allowed false positives} for expert $j$ in round $t$. The number of future allowed false positives of expert $j$ is the number of false positives that $j$ can make in the future and still being kept in the version space. Namely, it is $k$ minus the number of false positives that $j$ have already made. As usual, a false positive mistake of an expert $j$ refers to the scenario $\hat{y}^{(j)}_t = 1$ and $y_t = 0$. Other than that, all notation appearing in $\ExpAT$ has the same meaning as in $\RealizableExpAT$.

\begin{figure}
    \centering
    \begin{tcolorbox}
    \begin{center}
        \textsc{$\ExpAT$}
    \end{center}
    \textbf{Input:} Set of experts indexed by $[n]$, realizability parameter $k \in \mathbb{N}$, learning parameters $\eta \in (0,1), L > 1$.
    \\
    \textbf{Initialize:} Let $V_1 = [n]$, let $d^{(j)}_1 = 0$ and $k^{(j)}_1 = k$ for all $j \in [n]$.\\
    \\
    \textbf{for $t=1,\ldots, T$:}
    \begin{enumerate}
        \item Receive experts predictions $\hat{y}^{(1)}_t, \ldots, \hat{y}^{(n)}_t$.
        \item Predict $\hat{y}_t = 1$ if and only if
        \[
        \sum_{j \in V_t} \ind[\hat{y}^{(j)}_t = 1] 2^{k^{(j)}_t + \eta d^{(j)}_t} \geq L.
        \]
        \item If $\hat{y}_t = 1$:
        \begin{enumerate}
            \item Set $d^{(j)}_{t+1} = d^{(j)}_{t}$ for every expert $j\in V_t$.
            \item If $y_t = 0$: 
            \begin{enumerate}
                \item Set $V_{t+1} = V_t \backslash \{j \in V_t: \hat{y}^{(j)}_t = 1, k_t^{(j)} = 0\}$.
                \item Set $k^{(j)}_{t+1} = k^{(j)}_{t} - \ind[\hat{y}^{(j)}_t = 1]$ for every expert $j \in V_{t+1}$.
            \end{enumerate}            
        \end{enumerate}
        \item If $\hat{y}_t = 0$: set $d^{(j)}_{t+1} = d^{(j)}_{t} + \ind[\hat{y}^{(j)}_t = 1]$ and $k^{(j)}_{t+1} = k^{(j)}_{t}$ for every expert $j\in V_t$.
    \end{enumerate}
    \end{tcolorbox}
    \caption{A deterministic apple tasting learner for prediction with expert advice.} 
    \label{fig:ExpAT}
\end{figure}

We will now analyze the algorithm's mistake bound. Towards this end, we introduce some additional notation. Fix $T,k$, and a $k$-realizable sequence $S$ of experts' predictions and true labels of length $T$ that $\ExpAT$ is executed on. It is convinient to assume without loss of generality, that all experts are eventually removed from the version space. To make sure that this assumption does not weaken the adversary, we allow it to continue after the $T$ legitimate rounds, only if it causes the learner to make a false positive in each one of the rounds to follow the first $T$ rounds. For every expert $j$ and $\ell \in \{0, \ldots, k\}$, let $D_j^{(\ell)}$ be the distance of expert $j$ when it makes its $(\ell+1)$'th false positive. Let $N_{d}^{(\ell)} = \{j \in [n]: D_j^{(\ell)} \in I_d\}$. Namely, $N_{d}^{(\ell)}$ is the set of experts having distance in $I_d = [d/\eta, (d+1)/\eta]$ when they make their $(\ell+1)$'th false positive. Let $n_{d}^{(\ell)} = |N_{d}^{(\ell)}|$.

\begin{lemma} \label{lem:experts-agnostic-main}
    For any learning parameters $\eta \in (0,1), L > 1$ satisfying $n \leq \eta L T/ 2^{k+1}$, we have:
    \[
    M_- \leq \frac{\log L}{\eta} + k + 1, \quad \text{and} \quad M_+ \leq 200 \eta T.
    \]
\end{lemma}

As in the realizable case, the challenging part in proving Lemma~\ref{lem:experts-agnostic-main}, is to upper bound $M_+$, and the key ingredient in this bound is an inequality proved in Lemma~\ref{lem:experts-agnostic-helper} by forming lower and upper bounds on the \emph{total weighted distance gained} by experts in the version space. Formally, we will lower and upper bound the quantity
\[
\TWWDG := \sum_{t \in T_0} \sum_{j \in V_t} \ind[\hat{y}^{(j)}_t = 1] 2^{k^{(j)}_t + \eta d^{(j)}_t}.
\]
We replace the notation from $\TWDG$ in the realizable case to $\TWWDG$ in the agnostic case, because in the agnostic case the weight of an expert $j$ is also controlled by its number of allowed false positives $k_t^{(j)}$ in addition to its distance $d_t^{(j)}$. The following decomposition of $\TWWDG$ will be helpful. For every $\ell \in \{0, \ldots, k\}$, let 
\[
\TWWDG^{(\ell)} = \sum_{t \in T_0} \sum_{j \in V_t} \ind[\hat{y}^{(j)}_t = 1 \land k^{(j)}_t = k - \ell] 2^{k^{(j)}_t + \eta d^{(j)}_t}.
\]
That is, $\TWWDG^{(\ell)}$ is the total weighted distance gained by the experts in the portion of the game in which they have made exactly $\ell$ false positives.

\begin{lemma} \label{lem:experts-agnostic-helper}
    Suppose that $\eta \in (0,1), L >1$ and $n \leq \eta L T /2^{k+1}$. Then:
    \[
    \sum_{\ell = 0}^{k} \sum_{d \in \mathbb{N}} n_d^{(\ell)} 2^{k -\ell + d} \leq 200 \eta L T.
    \]
\end{lemma}

\begin{proof}
    The main part of the proof is to lower and upper bound $\TWWDG$. The upper bound is just as in the realizable case: by  definition of the algorithm, for every $t \in T_0$ it holds that
    \[
    \sum_{j \in V_t} \ind[\hat{y}^{(j)}_t = 1] 2^{k^{(j)}_t + \eta d^{(j)}_t} < L,
    \]
    which implies $\TWWDG \leq L T_0 \leq LT$. 
    Let's prove the lower bound, which is more challenging.  Fix $\ell \geq 0$. Decompose $[n]$ into two sets $G^{(\ell)}, B^{(\ell)}$, where $G^{(\ell)}$ is the set of all experts that gained at least $\frac{1}{2 \eta}$ distance while having made exactly $\ell$ false positives.  Therefore:

    \begin{align*}
        \TWWDG^{(\ell)}
        &\geq
        \sum_{j \in G^{(\ell)}} \sum_{D \in \{D_j^{(\ell - 1)}, \ldots, D_j^{(\ell)}\}} 2^{k - \ell + D} \\
        &\geq
        \sum_{j \in G^{(\ell)}} \frac{1}{2 \eta} 2^{k - \ell + \eta \mleft(D^{(\ell)}_j - \frac{1}{2 \eta} \mright)}\\
        &=
        2^{k - \ell - 3/2} \sum_{j \in G^{(\ell)}} \frac{1}{\eta} 2^{\eta D^{(\ell)}_j}.
    \end{align*}
    For the experts in $B^{(\ell)}$ we have:
    \[
    \sum_{j \in B^{(\ell)}} \frac{1}{\eta} 2^{\eta D^{(\ell)}_j}
    \leq
    \sum_{j \in B^{(\ell)}} \frac{1}{\eta} 2^{\eta \mleft(D^{(\ell - 1)}_j + \frac{1}{2 \eta} \mright)}
    =
    \sqrt{2} \sum_{j \in B^{(\ell)}} \frac{1}{\eta} 2^{\eta D^{(\ell - 1)}_j}
    \leq 
    \sqrt{2} \sum_{j \in [n]} \frac{1}{\eta} 2^{\eta D^{(\ell - 1)}_j}.
    \]
    Summing both inequalities above gives:
    \begin{equation} \label{eq:experts-agnostic-decomposition}
    \sum_{j \in [n]} \frac{1}{\eta} 2^{\eta D^{(\ell)}_j}
    \leq
    \TWWDG^{(\ell)}/ 2^{k - \ell - 3/2} + \sqrt{2} \sum_{j \in [n]} \frac{1}{\eta} 2^{\eta D^{(\ell - 1)}_j}.
    \end{equation}
    Summing over $\ell$ and incorporating the number of future allowed false positives, we have:
    \begin{align*}
    \sum_{\ell = 0}^{k} 2^{k - \ell} \mleft(\sum_{j \in [n]} \frac{1}{\eta} 2^{\eta D^{(\ell)}_j} \mright)
    &=
    2^k \sum_{j \in N_0^{(0)}} \frac{1}{\eta} 2^{\eta D^{(0)}_j}
    +
    2^k  \sum_{d \geq 1} \sum_{j \in N_{d}^{(0)}} \frac{1}{\eta} 2^{\eta D^{(0)}_j}
    +
    \sum_{\ell = 1}^{k} 2^{k - \ell} \mleft(\sum_{j \in [n]} \frac{1}{\eta} 2^{\eta D^{(\ell)}_j}\mright).
    \end{align*}   
    The first two summands are a decomposition of the case $\ell = 0$ to experts with distance in $I_0$ and to experts with distance in $I_d$ for $d >0$. The third summand handles $\ell \geq 1$. We prove an upper bound for each summand.
    We can upper bound the first summand as
    \[
    2^k \sum_{j \in N_0^{(0)}} \frac{1}{\eta} 2^{\eta D^{(0)}_j}
    \leq
    2^{k+1} n \frac{1}{\eta}
    \leq
    LT,
    \]
    by the assumption $n \leq \eta L T /2^{k+1}$.
    We can upper bound the second summand as follows:
    \[
    2^k \sum_{d \geq 1} \sum_{j \in N_{d}^{(0)}} \frac{1}{\eta} 2^{\eta D^{(0)}_j}
    \leq
    4 \cdot \sum_{d \geq 1} \sum_{j \in N_{d}^{(0)}} \frac{1}{\eta} 2^{k + d -  1}
    \leq
    4 \TWWDG^{(0)}
    \leq
    4 \TWWDG.
    \]
    By $\eqref{eq:experts-agnostic-decomposition}$, we can bound the third summand by:
    \begin{align*}
        \sum_{\ell = 1}^{k} 2^{k - \ell} \mleft(\sum_{j \in [n]} \frac{1}{\eta} 2^{\eta D^{(\ell)}_j}\mright)
        &\leq
        \sum_{\ell = 0}^{k-1} 2^{k - (\ell + 1)} \mleft(\TWWDG^{(\ell + 1)}/ 2^{k - (\ell + 1) - 3/2} + \sqrt{2} \sum_{j \in [n]} \frac{1}{\eta} 2^{\eta D^{(\ell)}_j} \mright) \\
        &\leq
        \sum_{\ell = 0}^{k-1} 2^{3/2}\TWWDG^{(\ell + 1)} +
        \sqrt{2} \sum_{\ell = 0}^{k-1} 2^{k - (\ell +1)} \mleft( \sum_{j \in [n]} \frac{1}{\eta} 2^{\eta D^{(\ell)}_j} \mright). \\
        & \leq
         4 \TWWDG  +
        \frac{\sqrt{2}}{2} \sum_{\ell = 0}^{k} 2^{k - \ell} \mleft( \sum_{j \in [n]} \frac{1}{\eta} 2^{\eta D^{(\ell)}_j} \mright).
    \end{align*}
    Summing the three bounds above together, we have:
    \begin{align*}
        \sum_{\ell = 0}^{k} 2^{k - \ell} \mleft(\sum_{j \in [n]} \frac{1}{\eta} 2^{\eta D^{(\ell)}_j} \mright)
        &\leq
        LT + 4 \TWWDG + 4 \TWWDG  +
        \frac{\sqrt{2}}{2} \sum_{\ell = 0}^{k} 2^{k - \ell} \mleft( \sum_{j \in [n]} \frac{1}{\eta} 2^{\eta D^{(\ell)}_j} \mright) \\
        &\leq
        9 LT + \frac{\sqrt{2}}{2} \sum_{\ell = 0}^{k} 2^{k - \ell} \mleft( \sum_{j \in [n]} \frac{1}{\eta} 2^{\eta D^{(\ell)}_j} \mright).
    \end{align*}
    Rearranging the inequality above gives:
    \[
        \frac{9}{1-\frac{\sqrt{2}}{2}} LT
        \geq
        \sum_{\ell = 0}^{k} 2^{k - \ell} \mleft(\sum_{j \in [n]} \frac{1}{\eta} 2^{\eta D^{(\ell)}_j} \mright)
        \geq
        \sum_{\ell = 0}^{k} \frac{1}{2} 2^{k - \ell} \mleft( \sum_{d \geq 1} n_d^{(\ell)} 2^{ d} \mright).
    \]
    By assumption:
    \[
    \sum_{\ell = 0}^{k} 2^{k - \ell}  n_0^{(\ell)}
    \leq
    2^{k+1} n
    \leq
    LT.
    \]
    Summing the two inequalities above gives the statement of the lemma.
\end{proof}

We may now prove Lemma~\ref{lem:experts-agnostic-main}.

\begin{proof}[Proof of Lemma~\ref{lem:experts-agnostic-main}]
    Let's first bound $M_-$. Let $j^\star$ be an expert who makes at most $k$ mistakes.
    In each of the rounds of $T_-$ except for at most $k$, the distance of $j^\star$ is increased by $1$. Therefore, if $|T_-| > \frac{\log L}{\eta} + k + 1$, then  $d_t^{(j^\star)} > \frac{\log L}{\eta} + k + 1$ for some round $t$.
    However, by the algorithm's definition, this means that there was some round $t > t' \in T_0$ for which $d_{t'}^{(j^\star)} > \log L/ \eta$ and $\hat{y}^{(j^\star)}_{t'} = 1$. This is not possible, since it would have imply $2^{\eta d_{t'}^{(j^\star)}} > L$ while $j^\star$ predicts $1$, which implies $\hat{y}_{t'} = 1$.

    Let's now bound $M_+$. Denote $m  =  M_+$ and $T_+ = \{t_1, \ldots, t_m\}$. For every round $t_i \in T_+$, the set $A_i$ of experts predicting $1$ must have sum of weights at least $L$. Therefore, we have:
    \[
    \sum_{i \in [m]} \sum_{j \in A_i} 2^{k^{(j)}_{t_i} + \eta D^{(j)}_{t_i}} \geq m L.
    \]
    On the other hand, the same quantity is upper bounded by:
    \begin{align*}
    \sum_{i \in [m]} \sum_{j \in A_i} 2^{k^{(j)}_{t_i} + \eta D^{(j)}_{t_i}}
    & =
    \sum_{\ell =0}^{k} \sum_{j \in [n]} 2^{k - \ell + \eta D^{(\ell)_j}} \\
    & \leq 
    \sum_{\ell = 0}^k \sum_{d \in \mathbb{N}} n_d^{(\ell)} 2^{k -\ell + d +1} \\
    & \leq
    200 \eta L T \tag{Lemma~\ref{lem:experts-agnostic-helper}}.
    \end{align*}
    Combining both equations above gives $m \leq 200 \eta T$ as required.
\end{proof}

It is now straightforward to infer Theorem~\ref{thm:experts-agnostic-upper-bound}.

\begin{proof}[Proof of Theorem~\ref{thm:experts-agnostic-upper-bound}]
    We run $\ExpAT$ with learning parameters $L = n \cdot 2^{k+1}$ and $\eta = \sqrt{\frac{k + \log n}{T}}$. The parameters are valid since $L > 1$ and $\eta \in (0,1)$. Further, we have $\eta L T/ 2^{k+1} = \sqrt{\frac{k + \log n}{T}} n  T \geq n$, since $\sqrt{\frac{k + \log n}{T}} T = \sqrt{T (k + \log n)} \geq 1$. Therefore Lemma~\ref{lem:experts-agnostic-main} holds and we have
    \[
    \M(\ExpAT, S) = M_- + M_+ \leq \frac{\log L}{\eta} +k  + 1 + 200 \eta T = \sqrt{T (k + \log n)} + k + 1 + 200\sqrt{T (k + \log n)}.
    \]
    Furthermore, it is clear that $\ExpAT$ is computationally efficient by its definition.
\end{proof}

\subsection{Prediction with expert advice: lower bounds} \label{sec:experts-lower-bounds}
In this section we prove lower bounds for the problem of prediction with expert advice, which will affirm that our algorithms are optimal (up to a constant multiplicative factor) for $n \geq T$. We start with a bound that holds already in the realizable case.

\begin{lemma} \label{lem:experts-realizable-lowerbound}
    For every $2 \leq T \leq n \leq 2^T$:
    \[
    \M^\star(n,T) = \Omega \mleft(\sqrt{T \log n} \mright).
    \]
\end{lemma}

\begin{proof}
   If $\sqrt{2}^T \leq  n \leq 2^T$ the lower bound is trivial. Suppose that $n < \sqrt{2}^T$.
   The adversary's strategy is to operate in phases, as explained below.
   Let $n_i$ be the number of consistent experts in the beginning of phase $i$ (initially we have $n_1 = n$). In phase $i$, as long as $n_i \geq \ceil*{\sqrt{T/ \log n}}$, the adversary splits the $n_i$ experts into $\ceil*{\sqrt{T/ \log n}}$ equal as possible blocks, each containing at most a $\sqrt{\log n / T}$-fraction of the $n_i$ experts (this is indeed a fraction since we assume $n < \sqrt{2}^T$). Denote the blocks by $b_1, \dots, b_{\ceil*{\sqrt{T/ \log n}}}$. The phase continues as long as the learner predicts $0$, and operates as follows. Starting with $j=1$, in each round of the phase the experts in $b_j$ predict $1$, and the rest predict $0$. At the end of the round, increase $j$ by $1$ if $j < \ceil*{\sqrt{T/ \log n}}$, and set $j=1$ otherwise. When the learner predicts $1$, the true label reported is $0$, and the adversary moves on to the next phase. Let $T_i$ be the number of rounds in phase $i$. Denote $T_i = t_i \ceil*{\sqrt{T/ \log n}} + r_i$ where $0 \leq r_i < \ceil*{\sqrt{T/ \log n}}$. In words, $t_i$ counts the minimal number of times that some expert predicts $1$ during phase $i$. If all experts but at most $\ceil*{\sqrt{T/ \log n}}$ are inconsistent before $T$ rounds have passed, it must hold that:
    \[
     e^{-2 \sqrt{\log n / T} P} n \leq e^{-\frac{\sqrt{\log n / T}}{1 - \sqrt{\log n/ T} }P} n \leq  (1- \sqrt{\log n /T})^P n \leq 2 \sqrt{T/ \log n},
    \]
    where the first and last inequalities are by the assumption $n < \sqrt{2}^T$, and the second inequality uses $1+x \geq e^{\frac{x}{1+x}}$ for $x> -1$. After rearranging and using $n \geq T$, the inequality above implies $P > \sqrt{T/ \log n}/8$. Therefore we are done, since the learner makes a mistake every time that a phase ends.
    So, we may now assume that all $T$ rounds of the game are played. Thus:
    \[
    T = \sum_{i=1}^P T_i = \sum_{i=1}^P \mleft( t_i \ceil*{\sqrt{T/ \log n}} + r_i \mright).
    \]
    If $P \geq \sqrt{T \log n} /4$ we are done, so assume that $P < \sqrt{T \log n} /4$. Therefore:
    \[
    \sum_{i=1}^P r_i < \frac{\sqrt{T \log n}}{4}  \cdot \ceil*{\sqrt{T/ \log n}} < T/2,
    \]
    where the second inequality is due to $n < \sqrt{2}^T$.
    By the two equations above, we have:
    \[
    \sum_{i=1}^P t_i \ceil*{\sqrt{T/ \log n}}  \geq T/2.
    \]
    Rearranging and using $n < 2^{\sqrt{T}}$ again gives:
    \[
    \sum_{i=1}^P t_i >\frac{T/2}{\ceil*{\sqrt{T/ \log n}}/2} \geq \sqrt{T \log n}/4.
    \]
    Every consistent expert predicted $1$ when the learner predicted $0$ for at least $\sum_{i=1}^P t_i$ many times. This finishes the proof by letting some consistent expert to determine the true labels.
\end{proof}

We now prove a lower bound for the agnostic case.
\begin{theorem} \label{thm:experts-agnostic-lowerbound}
    For every $2 \leq T \leq n < \sqrt{2}^T$ and $k \geq 0$:
    \[
    \M^\star(n,T, k) = \Omega(\sqrt{T (k + \log n)}).
    \] 
\end{theorem}

\begin{proof}
    The $\sqrt{T \log n}$ term is due to Lemma~\ref{lem:experts-realizable-lowerbound}. It remains to prove a lower bound of $\Omega(\sqrt{k T})$ when assuming $k\geq 1$. It is convenient (but not necessary) to assume that $\sqrt{T/k} \in \mathbb{N}$. The adversary's strategy proceeds in $\sqrt{T/k}$ phases, each of length $\sqrt{T k}$ rounds. In phase $i$, expert $i$ predicts $1$ and all other experts predict $0$. The adversary reports that the true label is $0$ every time the learner predicts $1$. If in some phase $i$ the learner predicted $1$ for less than $k$ many times, then we set the true label in all rounds of phase $i$ in which the learner predicted $0$ to be $1$. Therefore expert $i$ is $k$-consistent with the input, and the learner has made $\sqrt{T k}$ many mistakes. Otherwise, the learner predicts $1$ for at least $k$ many times in all $\sqrt{T/k}$ phases, which results in $\sqrt{Tk}$ many mistakes. Since there are more than $\sqrt{T/k}$ experts, there is an expert who is consistent with letting all true labels to be $0$.
\end{proof}

\section{Learning hypothesis classes}\label{sec:hypothesis-classes}
In this section we prove our results for learning hypothesis classes. Towards this end, we informally recall the  definition of the \emph{effective width} of a class $\cH$, formally defined in \cite{raman2024apple}, and denoted by $\w(\cH)$. We assume that $\LD(\cH) \geq 1$, as otherwise $|\cH| \leq 1$. For any natural $w,d$ so that $d \geq w$, define a tree of width $w$ and depth $d$ as follows. Start from a perfect tree of depth $d$. Traverse every branch in the tree, while counting right edges. Once the $w$'th right edge of a branch is reached, remove the entire subtree beneath it, except for its root, which now becomes a leaf. For every $w\geq 1$, $D_w(\cH)$ is defined to be the maximal $d \geq w$ such that there exists a tree of width $w$ and depth $d$ which is shattered  by $\cH$. If there is no such tree even for $d=w$, then $D_w(\cH) = 0$. If there are such trees with arbitrarily large depth then $D_w(\cH) = \infty$. The effective width $\w(\cH)$ is defined to be the minimal $w$ so that $D_w(\cH) < \infty$. If there is no such minimal value then $\w(\cH) = \infty$.

It was proved in \cite{raman2024apple} that $\w(\cH)$ controls the randomized mistake bound of $\cH$ under apple tasting feedback in the realizable setting. In this section we prove that it controls both its deterministic and randomized mistake bounds, even in the agnostic setting.

\subsection{Learning hypothesis classes: Upper bounds}
In this section, we prove the following upper bounds.
\begin{theorem} \label{thm:concept-upper-bound}
    Let $\cH$ be a class, $k$ be a realizability parameter and $T$ be the time horizon.
    \begin{enumerate}
        \item If $\w(\cH) = 1$, then:
        \[
            \Mis(\cH,T,k) = O \mleft(D_1(\cH) (k + 1) \mright).
        \]
        \item If $1 < \w(\cH) < \infty$, then:
        \[
            \Mis(\cH,T,k) = O\mleft(\sqrt{T (k + \LD(\cH) \log T)} \mright).
        \]
    \end{enumerate}
\end{theorem}
Lemmas~\ref{lem:easy-concept-upperbound},\ref{lem:hard-concept-upperbound}, to be proved in the following subsections, imply Theorem~\ref{thm:concept-upper-bound}. In Section~\ref{sec:concept-lowerbound} we prove that for every class, the first item is tight as long as $T$ is sufficiently large, and the second item is tight up to a logarithmic factor of $T$ and constants depending on the class.

\subsubsection{Easy classes}
In this section we handle the case $\w(\cH) = 1$. Towards this end, we provide an extension of the deterministic algorithm of \cite{raman2024apple} for the case $\w(\cH) = 1$ to the $k$-realizable setting.

\begin{lemma} \label{lem:easy-concept-upperbound}
    Let $\cH$ be a class with $\w(\cH) = 1$. Then for all $k, T$:
    \[
    \Mis(\cH,T,k) = O \mleft(D_1(\cH) (k + 1) \mright).
    \]
\end{lemma}

In Section~\ref{sec:concept-lowerbound} we show that this mistake bound is tight for all easy classes, as long as $T$ is sufficiently large.
The idea of the upper bound is to use a variation of the budgeted concept classes technique of \cite{filmus2023optimal}.

\begin{definition}[false positive budgeted concept class]
    A \emph{false positive budgeted concept class} (or \emph{budgeted class}, for short) is a collection $\cB$ of pairs $(h, b)$ where $h\colon \cX \to \{0,1\}$ is a hypothesis and $b \in \mathbb{N}$ is the allowed number of false positives. Furthermore, all hypotheses are distinct.
\end{definition}

An input sequence $S = (x_1, y_1), \ldots, (x_T,y_T)$ is realizable by a budgeted class $\cB$ if there exists $(h,b) \in \cB$ so that $h(x_t) = 1$ for for all $y_t=1$, and $h(x_t) = 1$ for at most $b$ many indices $t$ with $y_t = 0$. A tree is shattered by $\cB$ if every branch in it is realizable by $\cB$. Given a concept class $\cH \subset \{0,1\}^\cX$, $x\in \cX$ and $y \in \{0,1\}$, we defined $\cH^{(x \to y)} = \{h\in \cH: h(x) = y\}$. For budgeted classes, we define
\begin{gather*}
        \cB^{(x \to 1)} = \{(h,b) \in \cB: h(x) = 1\}, \\ \cB^{(x \to 0)} = \{(h,b) \in \cB: h(x) = 0\} \cup \{(h,b-1): (h,b) \in \cB, b \geq 1, h(x) = 1\}.
\end{gather*}

In words, $\cB^{(x \to 1)}$ removes from $\cB$ all pairs with $h(x) = 0$. $\cB^{(x \to 0)}$ decreases the budget of  every hypothesis $h$ with $h(x) = 1$, or removes it from $\cB$ if it is out of budget. We now define an appropriate notion of $k$-shattering. A tree is $k$-shattered by a class $\cH$ if for every branch, there exists a hypothesis that agrees with the branch on all of its right edges, and on all but at most $k$ of its left edges. We can naturally simulate $k$-shattering using budgeted classes. For every hypothesis class $\cH$ and $k \in \mathbb{N}$, define $\cB_{\cH,k} = \{(h,k): h\in \cH\}$.

\begin{observation}
    Let $\cH$ be a hypothesis class. A tree is $k$-shattered by $\cH$ if and only if it is shattered by $\cB_{\cH,k}$.
\end{observation} 

For any class $\cH$, define $D_w^{(k)}(\cH)$ to be the maximal $d$ such that there exists a tree of width $w$ and depth $d$ that is $k$-shattered by $\cH$, or, equivalently, shattered by $\cB_{\cH,k}$.
For budgeted classes in general,  we use the standard notation $D_w(\cB)$ for the maximal $d$ such that there exists a tree of width $w$ and depth $d$ that is shattered by $\cB$. We may use the notation $h \in \cB$ to indicate that $(h,b) \in \cB$ for some $b \geq 0$.
Our algorithm makes at most $D_1^{(k)}(\cH)$ false positives and at most $k$ false negatives. It is presented in Figure~\ref{fig:NarrowConceptAT}.

\begin{figure}
    \centering
    \begin{tcolorbox}
    \begin{center}
        \textsc{$\NarrowConceptAT$}
    \end{center}
    \textbf{Input:} Class $\cH$ with $\w(\cH) = 1$, realizability parameter $k \in \mathbb{N}$.
    \\
    \textbf{Initialize:} Let $V_1 = \cB_{\cH,k}$.\\
    \\
    \textbf{for $t=1,\ldots, T$:}
    \begin{enumerate}
        \item Receive instance $x_t$.
        \item If there exists $h \in V_t$ so that $h(x_t) = 1$: predict $\hat{y}_t = 1$. Else, predict $\hat{y}_t = 0$.
        \item If $\hat{y}_t = 1$: Update $V_{t+1} = V_t^{(x_t \to y_t)}$
    \end{enumerate}
    \end{tcolorbox}
    \caption{A deterministic apple tasting learner for concept classes with $\w(\cH) = 1$.} 
    \label{fig:NarrowConceptAT}
\end{figure}

\begin{lemma} \label{lem:easy-learner}
    For every class $\cH$, $k \in \mathbb{N}$, and a $k$-realizable input sequence $S$:
    \[
    \M^\star(\NarrowConceptAT, S) \leq D_1^{(k)}(\cH) + k
    \]
\end{lemma}

\begin{proof}
    We will show that $\NarrowConceptAT$ makes at most $k$ false negatives and at most $D_1^{(k)}(\cH) $ false positives. Let us start with the false negatives.
    $\NarrowConceptAT$ makes a false negative exactly in rounds $t$ where all hypotheses in $V_t$ agree on the label $0$, but the correct label is $1$.
    For every round $t$, the hypotheses in $V_t$ are exactly those who have made at most $k$ false positives. By the $k$-realizability guarantee, at all times there must be some $h^\star \in V_t$ that made at most $k$ many mistakes, and thus there could be at most $k$ many rounds in which all hypotheses in $V_t$ agree on the label $0$, and yet the correct label is $1$.

    Let us now show that $\NarrowConceptAT$ makes at most $D_1^{(k)}(\cH)$ false positives. Let $t$ be a round in which $\NarrowConceptAT$ makes a false positive. Suppose that $D_1(V_t) = D_1(V_{t+1})$. Then we can reach a contradiction by constructing a tree of depth $D_1(V_t) + 1$ and width $1$ that is shattered by $V_t$, as follows. Let $r$ be the root of this tree, labeled with $x_t$. By the algorithm's definition, there must be $h \in \cB$ so that $h(x_t) = 1$, so we add a right edge emanating from $r$, and a leaf labeled by $h$ beneath this edge. We also add a left edge emanating from $r$. Beneath this edge, we attach a tree of width $1$ and depth $D_1(V_t) = D_1(V_{t+1})$ which is shattered by $V_{t+1} = V_t^{(x_t \to 0)}$, by assumption. We conclude that $D_1(V_t)$ drops by at least $1$ after each false positive. Note that $D_1(V_t) = 0$ implies that for every $x$, all hypotheses in $V_t$ agree on the same label. If that label is $1$, then all hypotheses have made at least $k$ many false positives, thus $1$ must be the correct label. If this label is $0$, a false positive cannot be made, by the algorithm's definition. Therefore, at most $D_1^{(k)}(\cH)$ false positives are made.
\end{proof}

In order to prove Lemma~\ref{lem:easy-concept-upperbound}, it remains to upper bound $D_1^{(k)}(\cH)$.

\begin{lemma} \label{lem:bounded-k-shattered}
    Let $\cH$ be a class with $\w(\cH)=1$. We have
    \[
    D_1^{(k)}(\cH) = O(D_1(\cH) (k + 1)).
    \]
\end{lemma}

\begin{proof}
    If $k=0$ then the lemma is trivial, so suppose that $k >0$ and that $D_1^{(k)}(\cH) \geq 4 D_1(\cH) \cdot k$.  Let $\tree$ be a tree witnessing $D_1^{(k)}(\cH)$. Let $\ell_1$ be the lowest leaf that is right beneath a right edge in $\tree$, and let $h$ be the hypothesis labeling it. Let $r_1$ be the parent vertex of $\ell_1$. For every instance $x$ labeling an internal node $v$ that is above $r_1$ in $\tree$, and such that $h(x) = 1$,  we conduct the following process. Remove from $\tree$ the node $v$ labeled by $x$, it's right outgoing edge and leaf beneath it, and it's left outgoing edge $e$. Let $v'$ be the node beneath the left edge of $v$, and let $e'$ be the ingoing edge of $v$, if exists. We naturally reconnect the tree by connecting $e'$ as an ingoing edge to $v'$. Note that every hypothesis $k$-realizing a branch in $\tree$ has at most $k$ disagreements on left edges of the branch, and therefore we overall decreased the depth of $\tree$ by at most $k$. Now, if exists, let $r_2$ be the node above $r_1$ (after all manipulations made), and let $\ell_2$ be the right child of $r_2$. We now perform exactly the same process made before for $\ell_1$, on $\ell_2$. We repeat this process also for $r_3, r_4 \ldots, r_{D}$, as many times as possible, until we reach to $r_D$, which is the root of the tree (after the manipulations made). We argue that $D \geq 2 D_1(\cH)$. Indeed, every treatment of $r_i$ decreases the depth of $\tree$ by at most $k$. By assumption, $D_1^{(k)}(\cH) \geq 4 D_1(\cH) \cdot k$, and therefore after decreasing it by $2 k D_1(\cH) $, there are still internal vertices above $r_{2 D_1(\cH)}$ to handle.

    When  the process is finished, we have a tree of width $1$ and depth $D \geq 2 D_1(\cH)$ which is shattered by $\cH$, and that is a contradiction.
\end{proof}

Lemma~\ref{lem:hard-concept-upperbound} follows from plugging the bound of Lemma~\ref{lem:bounded-k-shattered} to the bound of Lemma~\ref{lem:easy-learner}.

\subsubsection{Hard classes}

\begin{lemma} \label{lem:hard-concept-upperbound}
    Let $\cH$ be a class with $1 < \w(\cH) < \infty$. Then for all $k,T$:
    \[
    \Mis(\cH,T,k) = O\mleft(\sqrt{T (k + \LD(\cH) \log T)} \mright).
    \]
\end{lemma}

The idea is to use the covering technique from \cite{bendavid2009agnostic} with our bounds for \emph{prediction with expert advice}. Formally, we say that an expert (or a learning algorithm) indexed by $i$ covers a hypothesis $h$ with respect to a sequence of instances $S = x_1, \ldots, x_T \in \cX$ if $\hat{y}_t^{(i)} = h(x_t)$ for all $t \in [T]$. We say that a class $\cH$ is covered by a set of experts with respect to $S$ if every $h \in \cH$ has a covering expert with respect to $S$ within the set. The work of \cite{bendavid2009agnostic} proves the following.

\begin{theorem}[\cite{bendavid2009agnostic}]  \label{thm:covering}
    Let $\cH$ be a class with $\LD(\cH) < \infty$, and let $T \in \mathbb{N}$. There exists a class of $n \leq T^{\LD(\cH)}$ experts that covers $\cH$ with respect to any sequence of instances of length at most $T$.
\end{theorem}
We may now prove the upper bound.
\begin{proof}[Proof of Lemma~\ref{lem:hard-concept-upperbound}]
    As argued in \cite{raman2024apple}, $\w(\cH) < \infty$ implies $\LD(\cH) < \infty$. By Theorem~\ref{thm:covering}, we can  cover $\cH$ with respect to any sequence of length at most $T$ with at most $T^{\LD(\cH)}$ experts. Since $S$ is $k$-realizable by $\cH$, the best  expert makes at most $k$ many mistakes on $S$. Thus, Theorem~\ref{thm:experts-agnostic-upper-bound} implies that
    \[
    \Mis(\cH,T,k) = O \mleft( \sqrt{T \mleft(k + \log T^{\LD(\cH)} \mright)} \mright),
    \]
    as required.
\end{proof}

\subsection{Learning hypothesis classes: universal lower bounds}\label{sec:concept-lowerbound}

In this brief section, we prove simple lower bounds that apply to all classes, and showing that our upper bounds are nearly tight. The unified adversarial  strategy is given in the following lemma.

\begin{lemma}\label{lem:concept-lowerbound-help}
    Let $\cH$ be a class, and let $k \leq T$. Let $D \in \mathbb{N}$ so that  $D \leq  \min \{D_1(\cH), \sqrt{T/(k+1)} \}$. Then
    \[
    \M^\star(\cH, T,  k) \geq  D (k+1).
    \]
\end{lemma}

\begin{proof}
    Let $\tree$ be a shattered tree of width $1$ and depth  $D$, which exists by assumption. The adversary's strategy is simple. Starting from the root, ask on any instance labeling the vertices of the branch  consists of only left edges, for $D (k+1)$ many times each. Always respond with $0$ as the true label. This amounts to a total of $D^2 (k+1) \leq T$ many  rounds by assumption, thus the strategy is implementable when there are $T$ many rounds. If for some instance $x$ the learner predicts $1$ for less than $k+1$ many times, stop. We determine the true labels by the branch that ends at the right edge of the vertex labelled by $x$. Therefore, the adversary adheres to the $k$-realizability constraint and the learner has made at least $D (k + 1)$ many mistakes. If, on the other hand, the learner predicts $1$ for more than $k$ many times on all instances, then we determine all true labels to be $0$. This choice is realized by the branch consists of only left edges. In this case, the learner has made at least $k+ 1$ many mistakes on every instance, and overall at least $D (k + 1)$ many mistakes. 
\end{proof}

The following theorem is an immediate corollary.

\begin{theorem} \label{thm:concept-lowerbound}
    Let $\cH$ be a class, and let $k \leq T$. Then:
    \begin{enumerate}
        \item If $\w(\cH) = 1$ and $T \geq (D_1(\cH))^2 (k+1)$ then:
        \[
        \M^\star(\cH, T,  k) \geq D_1(\cH) (k+1).
        \]
        \item If $\w(\cH) > 1$ then:
        \[
        \M^\star(\cH, T,  k) = \Omega \mleft( \sqrt{T (k+1)}  \mright).
        \]
    \end{enumerate}
    
\end{theorem}

\begin{proof}
The first item follows from applying Lemma~\ref{lem:concept-lowerbound-help} with $D = D_1(\cH)$. The second item follows from applying Lemma~\ref{lem:concept-lowerbound-help} with $D = \floor*{\sqrt{T/(k+1)}}$.
\end{proof}

\subsection{Trichotomy of mistake bounds} \label{sec:trichotomy}

In this section we prove that up to logarithmic factors and constants depending on the class, there is a trichotomy of mistake bounds for learning concept classes. Towards this end, we define $\M^\star_{\rand}(\cH, T, k)$ to be as $\M^\star(\cH, T, k)$, only that we also allow randomized learners. Respectively, $\M^\star_{\rand}(\cH, T, k)$ measures the optimal \emph{expected} number of mistakes, and not the number of mistakes, as $\M^\star(\cH, T, k)$.

\begin{theorem} \label{thm:trichotomy}
    Fix a  class $\cH$. Let $k \in \mathbb{N}$, and let $T \geq k$ be the time horizon. Then:
    \begin{enumerate}
        \item If $\w(\cH) = 1$:
        \[
        \M^\star_{\rand}(\cH, T, k), \M^\star(\cH, T, k) = \Theta_{\cH}(k).
        \]
        \item If $1 < \w(\cH) < \infty$:
        \[
        \M^\star_{\rand}(\cH, T, k) = \Theta(k) + \tilde{\Theta}_{\cH} \mleft(\sqrt{T} \mright), \quad  \M^\star(\cH, T, k) =\Theta \mleft(\sqrt{T k} \mright) + \tilde{\Theta}_{\cH} \mleft(\sqrt{T} \mright).
        \]
        \item If $\w(\cH) = \infty$:
        \[
        \M^\star_{\rand}(\cH, T, k) = T/2, \quad \M^\star(\cH, T, k) = T.
        \]
        The notation $\Theta_{\cH}(\cdot)$ hides constants depending on the class $\cH$.
    \end{enumerate}
\end{theorem}

\begin{proof}
    We prove every item separately, starting from the first item. The lower bound holds even for randomized learners with full information feedback. The deterministic upper bound is given in Lemma~\ref{lem:easy-concept-upperbound}.

    Let us prove the second item. The randomized upper bound is due to the agnostic regret bound given in \cite{raman2024apple}. In the randomized lower bound, the $\sqrt{T}$ term is due to the realizable lower bound of \cite{raman2024apple}, and the $k$ term is a lower bound even for $\w(\cH) = 1$. The deterministic upper bound is due to Lemma~\ref{lem:hard-concept-upperbound}, and the deterministic lower bound is due to Theorem~\ref{thm:concept-lowerbound}.

    In the third item, the upper bounds are trivial. The lower bounds are since $\w(\cH) = \infty$ implies $\LD(\cH) = \infty$, as noted in \cite{raman2024apple}. When   $\LD(\cH) = \infty$,  those lower bounds hold even with full-information feedback.  
\end{proof}

Note that Theorem~\ref{thm:trichotomy} implies a large separation between deterministic and randomized learners for every hard hypothesis class. For example, for any hard class $\cH$, if $k$ grows as $\sqrt{T}$, then the randomized mistake bound grows at most as $\sqrt{T \log T}$ and the deterministic mistake bound grows as $T^{3/4}$.

\subsection{Lower bounds for specific classes}\label{sec:concept-specific-lowerbound}

When considering logarithmic factors of $T$ and constants depending on the class, the deterministic bounds stated in Theorem~\ref{thm:trichotomy} are not tight for hard classes already in the realizable case. In fact, the accurate best known bounds for any hard class $\cH$ are:
\begin{equation} \label{eq:hard-concept-bounds}
    \Omega \mleft(\sqrt{\w(\cH) T} \mright) \leq \Mis(\cH, T) \leq O\mleft( \sqrt{ \LD(\cH) T  \log T} \mright),
\end{equation}
 where the lower bound is due to \cite{raman2024apple} and the upper bound is by Theorem~\ref{thm:concept-upper-bound}.

The goal in this section is to show that both inequalities in \eqref{eq:hard-concept-bounds} can be attained as equalities for infinitely many $\w(\cH), \LD(\cH)$ values, and therefore the bounds of \eqref{eq:hard-concept-bounds} are in fact the best possible bounds that hold for any hard class $\cH$, even when the effective width/Littlestone dimension is very large. The work of \cite{helmbold2000apple} describes a deterministic learner for the class of singletones over $\mathbb{N}$ that has a mistake bound of $O\mleft(\sqrt{T} \mright)$. This is rather straightforward to extend this result to the class $\cH_d$ of $d$-hamming balls over $\mathbb{N}$, for which the mistake bound will be $O \mleft(\sqrt{d T} \mright)$. It is also not hard to see that $\w(\cH_d) = d+1 = \LD(\cH) + 1$, and thus for every value of $\w(\cH)$ the left inequality can be attained as equality. In this section,  we will prove that the right inequality can also be attained for any $\LD(\cH)$ larger than some constant.

\begin{theorem} \label{thm:concept-realizable-lower-bound}
    For any natural $d$ larger than some universal constant there exists a class $\cH$ with $\LD(\cH) = O(d)$ and $T_0 = T_0(d)$, such that for every $T \geq T_0$:
    \[
    \Mis(\cH,T) = \Omega \mleft(\sqrt{d T \log T)} \mright).
    \]
\end{theorem}

To prove Theorem~\ref{thm:concept-realizable-lower-bound}, we first assume that $T$ is given and fixed, even before choosing the class $\cH$. We will later show how to remove this assumption by a simple ``gluing" technique that takes into account all values of $T$.
The class $\cU_n$  has $\LD(\cU_n) = \floor*{\log n}$. Therefore, using this class directly with the lower bound of Theorem~\ref{thm:intro-experts} will only give a lower bound of $\Omega\mleft( \sqrt{\LD(\cU_n) T}  \mright)$. However, looking into the proof of the lower bound in Theorem~\ref{thm:intro-experts} reveals that in contrast with full-information lower bounds, the ``hard" experts' predictions used by the adversary are very unbalanced, in the sense that many experts predict $0$ and only few predict $1$. Since $\LD(\cU_n) = \Omega(\log n)$ comes from choosing experts' predictions which are as balanced as possible, this intuitively means that we can remove instances from $\cX_n$ such that only unbalanced experts' predictions are available, in a way that decreases the Littlestone dimension from $\Omega(\log n)$ to $O(\log_T n)$, but maintains the $\Omega \mleft(\sqrt{T \log n} \mright)$ lower bound from Theorem~\ref{thm:intro-experts}. Choosing $n = T^d$ results in a class $\cH$ with $\LD(\cH) = O(d)$ that maintains a lower bound of $\Omega(T \log T^d)$.


We will now establish the existence of this class. Towards this end, we define a $(p,T,n)$-random class as follows. The domain is $\cX = \{x_1, \ldots, x_T\}$. The class consists of $n$ many hypotheses $\cH = \{h_1, \ldots, h_n\}$. Every hypothesis $h_i$ independently predicts $1$ on every instance $x_j$ with probability $p$.

\begin{lemma}\label{lem:rand-class}
    For every natural $d$ larger than some universal constant there exists $T_0(d)$ such that for every $T \geq T_0(d)$, there exists a hypothesis class $\cH(d,T) = \{h_1, \ldots, h_n\}$ of size $n=T^d$ over the domain $\cX = \{x_1, \ldots, x_T\}$ satisfying:
    \begin{enumerate}
        \item Every hypothesis in $\cH$ predicts $1$ on at least $\sqrt{d T \log T}$ instances.
        \item For a subset $X \subset \cX$, denote $H_{X \to 0}= \{h \in \cH: \forall x \in X, h(x) = 0\}$. For every subset $X \subset \cX$, if $\mleft|H_{X \to 0} \mright| \geq T^{d/2}$ then for every $x \in \cX$:
        \begin{equation} \label{eq:concept-hard-lowerbound-slow-decrease}
            \mleft| H_{X \to 0}^{(x \to 0)} \mright| \geq  \mleft(1- 1000\sqrt{\frac{d \log T}{T}} \mright) \mleft|H_{X \to 0} \mright|.
        \end{equation}
        \item $\LD(\cH) < 10 d$.   
    \end{enumerate}
\end{lemma}

\begin{proof}
    We prove the existence of $\cH:= \cH(d,T)$ by drawing a $(p,T,n)$-random class with $p = 100\sqrt{\frac{d \log T}{T}}$ and $n=T^d$ and prove that with positive probability, all items hold. Denote the event where Item~$i$ holds by $E_i$. Let us start with the first item. Fix a hypothesis $h$. By Chernoff's bound the probability that $h$ predicts $1$ for less than $\sqrt{d T \log T}$ many instances is at most
    \[
    e^{- 100 \sqrt{d T \log T} / 8} \leq e^{-10 \sqrt{T}}.
    \]
    A union bound now gives
    \begin{equation} \label{eq:E_1-bound}
        \Pr \mleft[\bar{E_1} \mright] \leq T^d/e^{10 \sqrt{T}}.
    \end{equation}

    Let us handle the second item. Fix $X \subset \cX$ and $x \in \cX$. If $x \in X$ then $H_{X \to 0}^{(x \to 0)} = H_{X \to 0}$, so in such a case \eqref{eq:concept-hard-lowerbound-slow-decrease} does not hold with probability $0$. If $x \notin X$, then since the predictions of hypotheses for different instances are independent, the probability that $\eqref{eq:concept-hard-lowerbound-slow-decrease}$ does not hold is at most
    \[
    e^{-T^{d/2} 100\sqrt{\frac{d \log T}{T}}} \leq e^{-T^{(d-1)/2}}
    \]
    by Chernoff's bound. By a union bound: 
    \begin{equation} \label{eq:E_2-bound}
        \Pr \mleft[ \bar{E_2} \mright] \leq \frac{2^{T + \log T}}{e^{T^{(d-1)/2}}}.
    \end{equation}

    Finally, we handle the third item. Let $\tree$ be a perfect tree of depth $10d$, with vertices labelled by instances from $\cX$. Fix a branch $b$ in $\tree$ with exactly $5d$ right edges. The probability that $b$ is realized by a fixed $h \in \cH$ is at most $\sqrt{\frac{d \log T}{T}}^{5d}$. Therefore the probability that $b$ is realized by $\cH$ is at most
    \begin{equation} \label{eq:branch-realized}
    1 - \mleft[ 1-  \sqrt{\frac{d \log T}{T}}^{5d}\mright]^{T^d}
    \leq
    2 T^d \sqrt{\frac{d \log T}{T}}^{ 5 d}
    =
    2 \frac{(d \log T)^{2.5d}}{T^{1.5d}}
    \leq
    1/T^{1.4d}.
    \end{equation}
    The first inequality is due to the Taylor series of $(1-\alpha)^\beta$ at $\alpha = 0$ implying $(1-\alpha)^\beta \leq  1-\alpha \beta + O \mleft((\alpha\beta)^2 \mright)$ for $\alpha \beta < 1$. The second inequality holds for sufficiently large $T_0$ since $T \geq T_0$.
    
    We now upper bound the probability that all $b$ are realized by $\cH$. Towards this end, we first show that the events where  different branches in $\tree$ are shattered by $\cH$ are negatively correlated. Let  $\cB:= \cB(\tree)$ be the set of all branches in $\tree$. For a class $\cH'$, let $\cS(\tree, \cH')$ be the set of all branches in $\tree$ which are realized by $\cH'$.
    Since $\cH$ is random and the predictions of its functions are chosen independently, the probability that a branch $b$ is realized by $\cH$ depends only on the size of $\cH$. Therefore, there  exists an increasing function $f \colon  \mathbb{N} \to [0,1]$ such that $\Pr[b \in \cS(\tree, \cH)] = f(|\cH|)$. Since every two branches disagree on at least one instance in $\cX$, it holds that if $b$ is realized by some function $h\in \cH$, then any other branch $b' \in \cB$ cannot be realized  by $h$. Thus, for any non-empty subset of branches $\cB' \subset \cB$ and a branch $b \notin B$, we have
    \[
    \Pr \mleft[b \in \cS(\tree, \cH) | \cB' \subset \cS(\tree, \cH) \mright] \leq f(|\cH| - |\cB'|).
    \]
    Therefore, by the conditional probability formula, for any subset $\cB' \subset \cB$:
    \begin{equation} \label{eq:-negative-correlation}
        \Pr[\cB' \subset \cS(\tree, \cH)] \leq \prod_{ b\in \cB'} \Pr[b \in \cS(\tree, \cH)].
    \end{equation}
    Let $\cB'$ be the set of branches with precisely $5d$ right edges. By $\eqref{eq:branch-realized}$, \eqref{eq:-negative-correlation} and $\binom{2n}{n} \geq 2^n/\sqrt{4n}$ for any $n$ larger than some universal constant:
    \[
    \Pr[\cB' \subset \cS(\tree, \cH)] \leq \mleft(1/T^{1.4d} \mright) ^{\binom{10d}{5d}}
    \leq
    1/T^{0.5 \sqrt{d}\cdot 2^{10d}}.
    \]
    On the other hand, there are at most $T^{2^{10d}}$ trees with vertices labeled by instances from $\cX$ of depth $10d$. A union bound gives:
    \begin{equation} \label{eq:E_3-bound}
        \Pr[\bar{E_3}]
        \leq
        \frac{T^{2^{10d}}}{T^{0.5 \sqrt{d}\cdot 2^{10d}}}
        = 1/ \sqrt{T}^{\sqrt{d}}.
    \end{equation}
    Finally, we deduce:
    \[
    \Pr[E_1 \cap E_2 \cap E_3]
    =
    1 - \Pr[\bar{E_1} \cup \bar{E_2} \cup \bar{E_3}]
    \geq
    1 - \sum_{i=1}^3 \Pr \mleft[ \bar{E_i} \mright]
    >
    0,
    \]
    where the first inequality is by a union bound, and the second inequality is by summing \ref{eq:E_1-bound}, \ref{eq:E_2-bound}, \ref{eq:E_3-bound} when $d$ is larger than some universal constant and $T \geq T_0$ where $T_0$ is sufficiently large.
\end{proof}

We may now prove Theorem~\ref{thm:concept-realizable-lower-bound} under the assumption that $T$ is fixed.
\begin{lemma} \label{lem:concept-realizable-lower-bound-fixedT}
    Fix $d$ larger than some universal constant, and $T \geq T_0$, where $T_0$ is as in Lemma~\ref{lem:rand-class}. There exists a class $\cH$ with $\LD(\cH) = O(d)$, such that:
    \[
    \Mis(\cH,T) = \Omega \mleft(\sqrt{d T \log T)} \mright).
    \]
\end{lemma}

\begin{proof}
    We choose the class $\cH:= \cH(d,T)$ guaranteed by Lemma~\ref{lem:rand-class}. The third item of Lemma~\ref{lem:rand-class} states that $\LD(\cH) \leq 10 d$. It remain to prove that $\Mis(\cH,T) = \Omega \mleft(\sqrt{d T \log T)} \mright)$. The adversary's strategy is very simple. It maintains a version space $V_t$ of hypotheses consistent with the feedback provided until (and include) round $t-1$. As long as $|V_t| \geq T^{d/2}$,  for every round $t$ where $\hat{y}_t = 1$, the true label is $y_t = 0$. Once $|V_t| < T^{d/2}$, the adversary chooses an arbitrary hypothesis $h^\star \in V_t$ and picks the true label to be $h^\star(x_t')$ for every round $t'$ larger than the first round $t$ in which $|V_t| < T^{d/2}$. This strategy guarantees that in any case, there exists $h^\star \in \cH$ which is consistent with the feedback.
    
    It remains to analyze the number of mistakes made by the learner. We have two cases. If $|V_T| \geq T^{d/2}$, we choose an arbitrary $h^\star \in V_T$ and determine $y_t = h^\star(x_t)$ for every round $t$. By definition of $V_T$, this choice of true labels is consistent with the feedback. By the first item of Lemma~\ref{lem:rand-class}, there are at  least $\sqrt{d T \log T}$ many rounds $t$ so that $h^\star(x_t) = 1$ and $\hat{y}_t = 0$, and therefore the learner makes at least $\sqrt{d T \log T}$ many mistakes. Otherwise, there exists some $t \in [T]$ such that $|V_{t}| \geq T^{d/2}$ and $|V_{t+1}| < T^{d/2}$. Denote the set of rounds in $\{1, \ldots, t\}$ where the learner predicted $1$ by $M = \{i_1, \ldots, i_m\}$ (where $i_m = t$). For every $j \in [m]$, denote $X_j = \{x_{i_1}, \ldots, x_{i_j}\}$. Note that by the adversary's strategy,  for every  $i_j \in M$, we have $V_{i_j + 1} = \cH_{X_j \to 0}$. The second item of Lemma~\ref{lem:rand-class} thus implies
    \[
    \mleft(1- 1000\sqrt{\frac{d \log T}{T}} \mright)^m T^d \leq T^{d/2}.
    \]
    Rearranging and using $1-x > e^{\frac{-x}{1-x}}$ for $x < 1$ gives $e^{1000\sqrt{\frac{d \log T}{T}} m} \geq T^{d/2}$ since $d,T$ are sufficiently large. After taking $\log$ of both sides, we obtain $1000\sqrt{\frac{d \log T}{T}} m \geq \frac{d}{2} \log T$, which finally gives:
    \[
    m \geq \sqrt{d T \log T}/2000.
    \]
    Since $m = |M|$ and the learner makes a mistake in every round from $M$, this finishes the proof.
\end{proof}

The final step remained for proving Theorem~\ref{thm:concept-realizable-lower-bound} is to remove the assumption that we know $T$ before choosing the class. This assumption can be removed by a rather simple technique, which we call \emph{gluing} concept classes. To the best of our knowledge, it was not defined in existing literature.

We define how to glue together two concept classes, and the definition is easily extendable to any countable collection of concept classes, by gluing two classes from the collection, and then gluing the obtained class again with another class, and so on.

\begin{definition}[Glued concept class]
Let $\cH_1$, $\cH_2$ be concept classes defined on finite, disjoint domains $\cX_1, \cX_2$. We call two classes satisfying these properties \emph{glueable classes}. Define the glued domain of $\cX_1, \cX_2$ to be $G(\cX_1, \cX_2) = \cX_1 \cup \cX_2$. Define the glued concept class of $\cH_1, \cH_2$, denoted by $G(\cH_1, \cH_2)$ as follows. Let $r \in \{1,2\}$. For every $h \in \cH_r$, the hypothesis $h^{(r)}$ belongs to $G(\cH_1, \cH_2)$, where:
\[
h^{(r)}(x) = 
\begin{cases}
h(x) & x \in \cX_r, \\
0 & x \notin \cX_r.
\end{cases}
\]
We say that the source of $h^{(r)}$ is $\cH_r$.
\end{definition}

For any countable collection of hypotheses classes $\{\cH_i\}_{i=1}^\infty$, if the domains of all classes are finite and pairwise disjoint, we say that the collection $\{\cH_i\}_{i=1}^\infty$ is glueable. The glued class $\cH = G\mleft(\{\cH_i\}_{i=1}^\infty\mright)$ is obtained by gluing $\cH_1$ and $\cH_2$, and the gluing $G(\cH_1,\cH_2)$ with $\cH_3$, and so on.

\begin{proposition}\label{prop:glue}
    Let $\{\cH_i\}_{i=1}^\infty$ be a countable collection of glueable classes, and let $\cH = G\mleft( \{\cH_i\}_{i=1}^\infty\mright)$. Then $\LD(\cH) \leq \max_{i}\LD(\cH_i) + 1$.
\end{proposition}

\begin{proof}
    We design an explicit learner for learning $\cH$ in the realizable, full-information feedback setting. The learner always predicts $0$ until it makes a mistake. Once that happens, the learner knows the source $\cH_i$ of the the target concept. Therefore it will make at most $\LD(\cH_i)$ more mistakes.
\end{proof}

We can now fully prove Theorem~\ref{thm:concept-realizable-lower-bound}.

\begin{proof}[Proof of Theorem~\ref{thm:concept-realizable-lower-bound}]
    Consider the collection $\{\cH(d,T)\}_{T \geq T_0}$, where $T_0$ is the same as in Lemma~\ref{lem:concept-realizable-lower-bound-fixedT}. For every $T$, we  differentiate the instances of $\cH(d,T)$ from the instances of other classes by relating to them as $x_1(T), \ldots, x_T(T)$, thus making the collection  $\{\cH(d,T)\}_{T \geq T_0}$ glueable. We can now define $\cH = G\mleft( \{\cH(d,T)\}_{T \geq T_0} \mright)$. By proposition~\ref{prop:glue} and Lemma~\ref{lem:rand-class}, $\LD(\cH) \leq 10d$. Now, given $T$, we use the adversary's strategy defined in Lemma~\ref{lem:concept-realizable-lower-bound-fixedT} for the class $\cH(d,T)$, which is consistent with $\cH$. Therefore $\Mis(\cH,T) = \Omega \mleft(\sqrt{d T \log T)} \mright)$ as required.
\end{proof}

\section{Prediction without prior knowledge}\label{sec:without-prior}
In this section, we remove the assumption that the number of rounds $T$ and the realizability parameter $k$ are given to the learner. We denote the unknown number of rounds and realizability parameter by $T^\star,k^\star$. The idea is similar to standard doubling tricks, such as the doubling trick of \cite{cesa1997use}.

\subsection{Prediction with expert advice without prior knowledge}

We present the doubling trick algorithm $\DTExpAT$ for prediction with expert advice in Figure~\ref{fig:DTExpAT}.

\begin{figure}
    \centering
    \begin{tcolorbox}
    \begin{center}
        \textsc{$\DTExpAT$}
    \end{center}
    \textbf{Input:} Class of $n$ experts indexed by $[n]$.
    \\
    \textbf{Initialize:} Let $g_k=g_T = 1$, $T_0 = 0$.\\
    \textbf{Denote:} $k = g_k \log n, T = g_T \log n$. \\
    \\
    \textbf{for $t=1,\ldots, T^\star$:}
    \begin{enumerate}
        \item If $t > T-T_0$, or, all experts have made more than $k$ false positives:
        \begin{enumerate}
            \item If $t > T-T_0$: Double $g_T$.
            \item If all experts have made more than $k$ false positives: Double $g_k$.
            \item Restart $\ExpAT$: Set $T_0 = t$, zero the false positives count of all experts, remove all information gathered in previous rounds from its memory.
        \end{enumerate}
        \item Predict as $\ExpAT$ predicts under the assumption that the best expert makes at most $k$ many mistakes and that there are $T$ many rounds, given all information gathered by $\ExpAT$ in previous rounds. Receive feedback when the prediction is $1$.
    \end{enumerate}
    \end{tcolorbox}
    \caption{A deterministic apple tasting learner for prediction with expert advice without prior knowledge.} 
    \label{fig:DTExpAT}
\end{figure}

\begin{theorem} \label{thm:prior-experts}
    Theorem~\ref{thm:experts-agnostic-upper-bound} holds even without assuming that the number of rounds $T^\star$ and the realizability parameter $k^\star$ are given to the learner. 
\end{theorem}

\begin{proof}
    Let $S$ be the input sequence. Let $i^\star$ be a target expert who makes at most $k^\star$ many mistakes. It is convenient to assume that $i^\star$ makes no false negative mistakes. If a mistake bound $M$ holds under this assumption, then a mistake bound of $M + k^\star$ holds without this assumption: the adversary may choose at most $k^\star$ rounds in which $\DTExpAT$ and $i^\star$ predict $0$, and let the true label in those rounds to be $1$. Since the mistake bound proved in Theorem~\ref{thm:experts-agnostic-upper-bound} (in which we ``compete") is larger than $\sqrt{T^\star k^\star} \geq k^\star$, adding $k^\star$ to $M$ does not change the bound in more than a constant factor.
    
    We use the learner $\DTExpAT$ described in Figure~\ref{fig:DTExpAT}. Let $g_T^\star, g_k^\star$ such that $T^\star =  g_T^\star \log n$ and $k^\star =  g_k^\star \log n$. By the assumptions in Thoeorem~\ref{thm:experts-agnostic-upper-bound}, $g_T^\star \geq 1$. We consider two cases. In the first case, $g_k^\star < 1$. This means that $g_k$ is never doubled, and that the mistake bound given in Thoeorem~\ref{thm:experts-agnostic-upper-bound} is $O \mleft( \sqrt{T \log n}\mright)$, which is the same as $O \mleft( \sqrt{g^\star_T} \log n\mright)$. Consider the following partition of $[T^\star]$ to intervals. Every interval in the partition ends at the round before the round where $g_T$ is doubled, and then a new interval begins. Let $\cI_T = \{I_0, \ldots, I_{T'}\}$ be the set of intervals in the partition, Where $I_i$ is the interval in which the value of $g_T$ is $2^i$. For every $I \in \cI_T$, let $\Mis(I)$ be the number of mistakes made by $\DTExpAT$ in all rounds of interval $I$. Thus:
    \[
    \Mis(\DTExpAT, S)
    =
    \sum_{I_i \in \cI_T} \Mis(I_i)
    \leq 
    \sum_{I_i \in \cI_T} \sqrt{2^i 2\log^2 n}
    =
    \sqrt{2} \log n \sum_{I_i \in \cI_T} \sqrt{2^{i}}
    \leq
    10 \sqrt{g_T^\star} \log n.
    \]
    The first inequality is by the mistake bound of Theorem~\ref{thm:experts-agnostic-upper-bound}. The second inequality is since $\sum_{i=0}^{T'} \sqrt{2^i} \leq 4 \sqrt{2^{T'}}$ and because at all times, $g_T\leq 2g_T^\star$.
    
    We now handle the second case in which $g_k^\star \geq 1$. In this case, the mistake bound given in Theorem~\ref{thm:experts-agnostic-upper-bound} is $O \mleft( \sqrt{T k} \mright)$, which is the same as $ O \mleft(\sqrt{g_T^\star g_k^\star} \log  n \mright)$.
    In this case as well, we consider a partition of $[T^\star]$ to intervals. Every interval in the partition ends one round before at least one of $g_k, g_T$ is doubled, and then a new interval begins. Let $\cI_k = \{I_1, \ldots, I_{k'}\}$ be the set of intervals initiated by doubling $g_k$, where $I_i$ is the interval in which the value of $g_k$ is $2^i$. For every $I \in \cI_k$, let $\Mis(I)$ be the number of mistakes made by $\DTExpAT$ in all rounds of interval $I$. We have:
    \[
    \sum_{I_i \in \cI_k} \Mis(I)
    \leq 
    \sum_{I_i \in \cI_k} \sqrt{2 g_T^\star \log n (2^i  \log n + \log n)}
    \leq
    2 \sqrt{g_T^\star} \log n \sum_{I_i \in \cI_k} \sqrt{2^{i}}
    \leq
    20 \sqrt{g_T^\star g_k^\star} \log n.
    \]
    The first inequality is since $g_T\leq 2g_T^\star$ at all times, and by the guarantees of Theorem~\ref{thm:experts-agnostic-upper-bound}.
    The last inequality is since $\sum_{i=1}^{k'} \sqrt{2^i} \leq 4 \sqrt{2^{k'}}$. We conduct roughly the same analysis for $\cI_T = \{I_1, \ldots, I_{T'}\}$ which is the set of intervals initiated by doubling $g_T$, Where $I_i$ is the interval in which the value of $g_T$ is $2^i$. This gives:
    \[
    \sum_{I_i \in \cI_T} \Mis(I)
    \leq 
    \sum_{I_i \in \cI_T} \sqrt{2^i \log n (2 g_k^\star  \log n + \log n)}
    \leq
    2 \sqrt{g_k^\star} \log n \sum_{I_i \in \cI_T} \sqrt{2^{i}}
    \leq
    20 \sqrt{g_T^\star g_k^\star} \log n.
    \]
    The first inequality is since $g_k\leq 2g_k^\star$ at all times.
    For the first interval before either $g_k$ or $g_T$ is doubled, $O(\log n)$ mistakes are made. Summing the total number of mistakes made in all rounds, we get a mistake bound of $\Mis(\DTExpAT, S) \leq O\mleft( \sqrt{g_T^\star g_k^\star} \log n \mright)$, as required.
\end{proof}

\subsection{Learning hypothesis classes without prior knowledge}

\begin{theorem}
    Theorem~\ref{thm:concept-upper-bound} holds even without assuming that the number of rounds $T^\star$ and the realizability parameter $k^\star$ are given to the learner. 
\end{theorem}

\begin{proof}
    We first prove the bound for hard classes, and then for easy classes.

    Our bound for a hard class $\cH$ is obtained via a reduction to prediction with expert advice, thus it is almost completely handled by Theorem~\ref{thm:prior-experts}. The only difference is that in the reduction, the number of experts depends on $T$. Thus, when $\DTExpAT$ doubles the guess of $T$ to $2T$, it also needs to enlarge the number of experts from $T^{\LD(\cH)}$ to $(2T)^{\LD(\cH)}$. Since the final guess of $T$ is at most $2 T^\star$, we obtain the upper bound $O \mleft(\sqrt{T^\star \mleft(k^\star + \log (2 T^\star)^{\LD(\cH)} \mright)} \mright)$, which is  $O \mleft(\sqrt{T^\star (k^\star +  \LD(\cH) \log T^\star)} \mright)$, as required.

    For easy classes, the mistake bound achieved by $\NarrowConceptAT$ does not depend on $T^\star$, but it does use knowledge of $k^\star$. To remove the assumption that $k^\star$ is given, it is convenient to assume, as we did for prediction with expert advice, that the best hypothesis makes no false negative mistakes. Removing this assumption will add a $k^\star$ term to the mistake bound, which does not increase it by more than a constant factor. Now, we initialize $k=1$ and run $\NarrowConceptAT$ with $k$ as the realizability parameter. If at some point the version space becomes empty, then the guess of $k$ is too small, and we double it. The final guess of $k$ will be at most $2k ^\star$. Therefore the total number of mistakes will be at most
    \[
    \sum_{i=1}^{ \ceil*{\log 2 k^\star}} D_1^{(2^i)}(\cH)
    \leq
    \sum_{i=1}^{ \ceil*{\log 2 k^\star}} O(D_1(\cH) + 1) 2^i
    =
    O((D_1(\cH) + 1) (k^\star + 1)),
    \]
    where the inequality is by Lemma~\ref{lem:bounded-k-shattered}, and the last equality is since $\sum_{i=1}^n 2^i = O(2^n)$.
\end{proof}

\section{Future work}

\paragraph{Explicit construction for Lemma~\ref{lem:rand-class}.}
In Lemma~\ref{lem:rand-class}, we give a non-constructive proof of existence of the class used for the lower bound of Theorem~\ref{thm:concept-realizable-lower-bound}. It will be interesting to construct this class explicitly. Note that the lower bound for prediction with expert advice (Lemma~\ref{lem:experts-realizable-lowerbound}) is explicit and uses a similar technique. However, in this bound we do not require that the Littlestone dimension of the class will be only $O(\log_T n)$, so it is simpler.

\paragraph{Relation to standard agnostic online learning.}
The techniques used to prove the upper bound for prediction with expert advice in the realizable case are similar to techniques used for proving upper bounds on the regret in the problem of agnostic prediction with expert advice under full-information feedback. Remarkably, the optimal bounds are also identical, up to  constant factors. Are the problems related? For example, is there a reduction between them, at least in one direction?

\paragraph{Complete landscape of mistake bounds.}
Given a fixed hard class $\cH$, the bounds established in Theorem~\ref{thm:trichotomy} when $k=0$ are tight up to a $\sqrt{\log T}$ factor. Furthermore, in Section~\ref{sec:concept-specific-lowerbound} we show that both sides of Inequality~\eqref{eq:hard-concept-bounds} are attained by some hard classes. Is it true that every mistake bound between the lower and upper bound of \eqref{eq:hard-concept-bounds} can be attained? For example, is there a hard class $\cH$ that has mistake bound $\Theta_{\cH} \mleft( \sqrt{T \log \log T} \mright)$ in the realizable setting?

\paragraph{Prediction with slightly more than $\sqrt{T}$ many experts.}
We prove an upper  bound of $O \mleft( \sqrt{T \log n} \mright)$ for prediction with expert advice in the realizable setting. We also prove a matching lower bound of $\Omega \mleft( \sqrt{T \log n} \mright)$, but it holds only for $n \geq T^{1/2 + \epsilon}$, where $\epsilon > 0 $ is a universal constant. On the other hand, it is not hard to see that if $n \leq \sqrt{T}$ then the optimal mistake bound is $\Theta \mleft( n \mright)$. The regime where $n = T^{1/2 + \epsilon}$ and $\epsilon = o(1)$ thus remains open.

\section*{Acknowledgments}
We would like to thank Vinod Raman and Unique Subedi for introducing the problem of deterministic apple tasting to us.

\bibliographystyle{alphaurl}
\bibliography{bib.bib}

\end{document}